\documentclass[12pt]{article}

\usepackage{graphicx}
\usepackage{natbib}
\usepackage{url}
\usepackage{amsmath, amssymb, amsthm, mathtools}
\usepackage{booktabs}
\usepackage{hyperref}
\hypersetup{
    colorlinks=true,
    urlcolor=blue,
    linkcolor=black
}
\usepackage{xcolor}
\usepackage{algorithm}
\usepackage{algpseudocode}
\usepackage{bm}
\usepackage{bbm}
\usepackage{setspace}

\newtheorem{theorem}{Theorem}
\newtheorem{proposition}{Proposition}

\newtheorem{lemma}{Lemma}
\newtheorem{assumption}{Assumption}
\newtheorem{definition}{Definition}
\newtheorem{remark}{Remark}

\usepackage{listings}
\usepackage{tcolorbox}
\tcbuselibrary{listings, breakable}

\definecolor{codebg}{RGB}{245,245,245}
\definecolor{codeframe}{RGB}{200,200,200}

\lstdefinestyle{custompython}{
    language=Python,
    basicstyle=\ttfamily\small,
    keywordstyle=\color{blue},
    commentstyle=\color{gray},
    stringstyle=\color{red},
    numbers=left,
    numberstyle=\tiny\color{gray},
    numbersep=8pt,
    showstringspaces=false,
    breaklines=true
}

\newtcblisting{codeblock}{
    colback=codebg,
    colframe=codeframe,
    listing only,
    listing options={style=custompython},
    breakable,
    sharp corners,
    boxrule=0.3pt,
    left=6pt,right=6pt,top=6pt,bottom=6pt
}

\newenvironment{proof*}{\begin{proof}}{\end{proof}}
\newtheorem*{lemma*}{Lemma}

\usepackage{fancyhdr}
\usepackage{tikz}

\definecolor{aoeblack}{RGB}{15,15,15}
\definecolor{aoegray}{RGB}{120,120,120}
\definecolor{aoered}{RGB}{200,30,30}

\newcommand{\blind}{0}

\begin{document}

\if0\blind
{
\begin{titlepage}
\thispagestyle{empty}

% ===== TOP BLACK BAR =====
\begin{tikzpicture}[remember picture, overlay]
  % FIXED: was 'ablack' (undefined), now correctly 'aoeblack'
  \fill[aoeblack]
    (current page.north west) rectangle
    ([yshift=-1.8cm]current page.north east);

  % Left: Logo
\node[anchor=west, xshift=1.2cm, yshift=-0.9cm]
  at (current page.north west)
  {\href{https://aliensonearth.in}{
    \includegraphics[height=1.0cm]{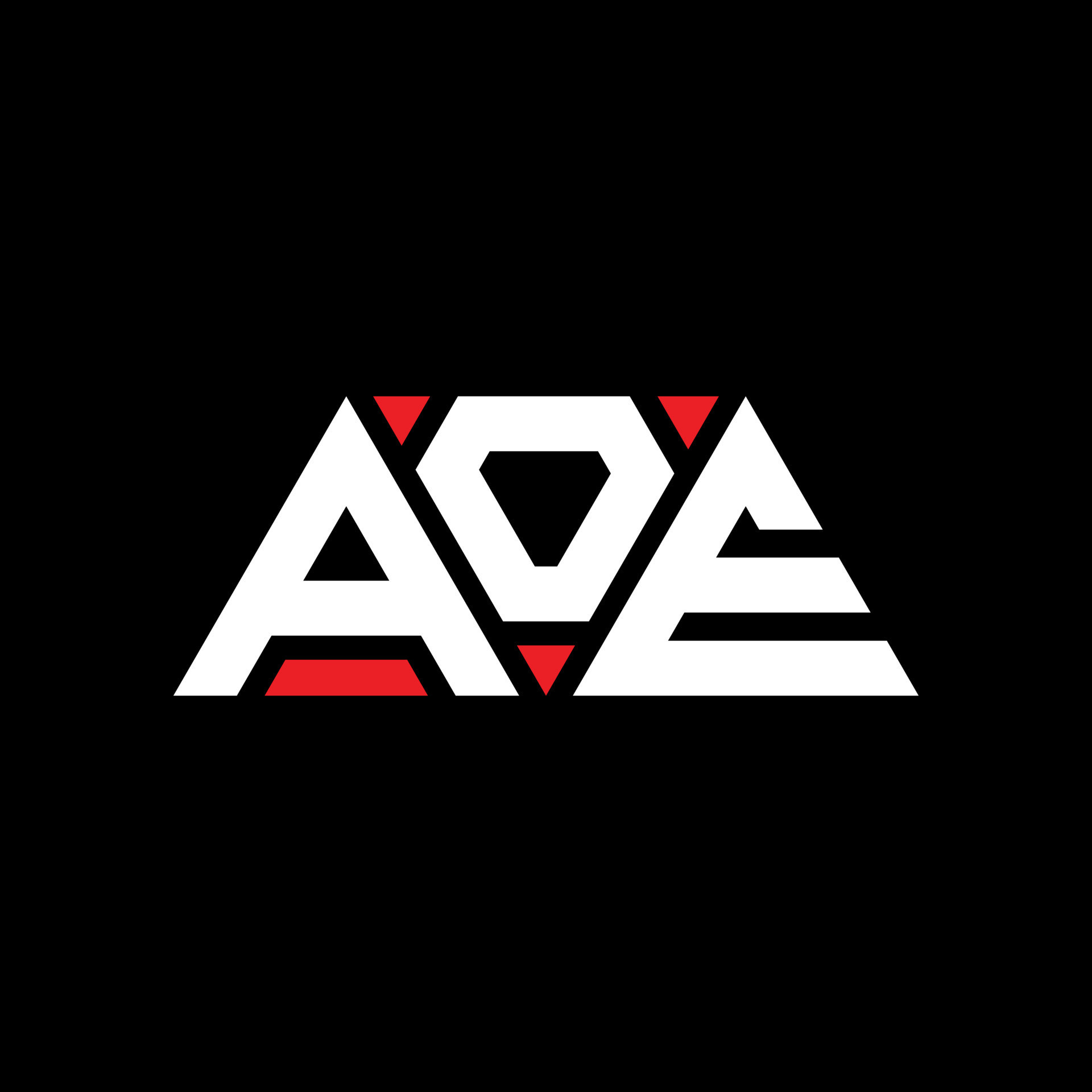}
  }};

  % Right: Label
\node[anchor=east, xshift=-1.2cm, yshift=-0.9cm]
  at (current page.north east)
  {\href{https://aliensonearth.in}{
    \color{white}\large \textsc{Aliens on Earth}
  }};
\end{tikzpicture}

\vspace*{3.2cm}
\centering

% ===== TITLE =====
{\LARGE \bfseries
Early Detection of Latent Microstructure Regimes\\[0.3cm]
in Limit Order Books
\par}

\vspace{4.0cm}

% ===== AUTHORS =====
{\large
\textbf{Prakul Sunil Hiremath}\\[-0.2em]
\texttt{prakulhiremath@vtu.ac.in}
\par}

\vspace{0.8cm}

{\large
\textbf{Vruksha Arun Hiremath}\\[-0.2em]
\texttt{vrukshahiremath@klecbalc.edu.in}
\par}

\vspace{0.8cm}

\vfill

{\color{black}\rule{0.22\textwidth}{1.0pt}}

\vspace{1.2cm}

% ===== FOOTER =====
{\small
Department of Computer Science and Engineering\\
Visvesvaraya Technological University, Belagavi\\[0.4cm]
Department of Business Administration\\
K.L.E. Society's College of Business Administration
}

\vspace{0.8cm}

{\small \color{aoegray} \today}

\end{titlepage}
}
\fi

%=============================================================================
\clearpage
\thispagestyle{fancy}

\vspace*{\fill}

\begin{center}

\begin{tcolorbox}[
  width=0.9\textwidth,
  colback=white,
  colframe=aoeblack,
  boxrule=0.6pt,
  arc=2mm,
  left=12pt,right=12pt,top=12pt,bottom=12pt
]

% ===== Top accent line =====
{\color{aoered}\rule{\textwidth}{1.2pt}}

\vspace{0.6em}

\begin{abstract}
Limit order books can transition abruptly from stable to stressed conditions, yet standard early-warning signals---such as order flow imbalance and short-term volatility---are reactive by construction: they respond to stress already underway rather than to its precursors.

We formalise this limitation via a three-regime causal data-generating process (stable $\to$ latent build-up $\to$ stress) in which stress onset is delayed relative to a preceding build-up phase, creating a prediction window. We show that the latent build-up regime is identifiable under mild conditions on temporal drift and regime persistence, and derive two complementary guarantees: (i) a sufficient drift-to-noise condition for strictly positive \emph{expected} lead-time (Proposition~\ref{prop:lead}), and (ii) a lower bound on the probability of detection before stress onset as a function of SNR and build-up duration (Proposition~\ref{prop:prob_early}). We provide a complete, rigorous proof of Proposition~\ref{prop:prob_early} that carefully distinguishes the CUSUM stopping time from the CUSUM partial sum, closing the gap between the concentration argument and the detection claim.

Building on this, we propose a trigger-based detector combining (i) MAX aggregation of uncertainty and drift channels, (ii) a rising-edge condition, and (iii) an adaptive threshold. Across 200 simulation runs, the method achieves mean lead-time $+18.6 \pm 3.2$ timesteps (95\% CI), precision $1.00 \pm 0.00$, and coverage $0.54 \pm 0.06$. We provide a conditional breakdown of the 46\% missed detection rate, showing it is concentrated in short build-up, high-noise parameter cells consistent with the theoretical bounds. Our method outperforms CUSUM, BOCPD, HMM thresholding, and imbalance/volatility baselines under the evaluated settings.

In a preliminary illustrative real-data application on one week of BTC/USDT order book data (1\,Hz, 5 labelled events), the detector achieves mean lead-time $+38 \pm 21$ seconds, precision $1.00$, and coverage $0.80$, while baselines exhibit negative lead-times. Depth erosion and HMM entropy account for over 99\% of first-trigger events, providing an interpretable mechanism consistent with the causal model.

These results are preliminary and based on a small number of events; performance degrades in low-SNR, short build-up regimes, consistent with the theoretical bounds.
\end{abstract}

\vspace{0.8em}


\end{tcolorbox}

\end{center}

\vspace{0.5em}
\vspace{0.6em}

\noindent{\scriptsize\color{aoegray}
\textit{Code \& data available at}
\;
\href{https://github.com/prakulhiremath/LOB-Latent-Regimes}{GitHub}
\;
·
\;
\href{https://doi.org/10.5281/zenodo.19697687}{Zenodo}
}

\vspace*{\fill}

\newpage
\setstretch{1.8}

%=============================================================================
\section{Introduction}
\label{sec:intro}
%=============================================================================

\paragraph{The problem: reactive signals in fast markets.}
Modern electronic limit order books (LOBs) can transition from abundant
liquidity to acute stress in seconds~\citep{Gould13}.
The practitioner toolkit for detecting such transitions---order flow
imbalance, bid--ask spread widening, volatility spikes---shares a
structural flaw: every signal in this toolkit is a \emph{consequence}
of stress, not its precursor.
A detector calibrated on imbalance or volatility will, by construction,
fire at time $\tau \geq \sigma$, where $\sigma$ is the onset of visible
stress.
Zero or negative lead-time is not a tuning failure; it is the logical
outcome of measuring observables that are \emph{defined} by the regime
you are trying to predict.

\paragraph{The gap: latent build-up is untapped.}
A natural question is whether stress is truly abrupt or whether it is
preceded by a latent deterioration phase invisible to threshold-based
detectors.
Microstructure theory provides intuition: market depth is consumed
gradually as informed traders accumulate positions
\citep{Kyle85, Bouchaud09}, order-flow toxicity builds before it is
reflected in prices \citep{Easley12}, and inventory imbalances at
market-makers widen spreads with delay \citep{Hasbrouck91}.
Yet no existing detection framework \emph{exploits} this temporal
structure to achieve systematically positive lead-time.

\paragraph{What we add.}
This paper provides four concrete contributions.
\begin{enumerate}
  \item \textbf{A causal regime model with formalised identifiability.}
  We construct a three-regime DGP---stable, latent build-up, stress---
  in which build-up produces persistent but weak drift in LOB features
  before stress materialises.
  We prove that the latent regime sequence is identifiable from
  observations under conditions on temporal drift and regime persistence
  (Theorem~\ref{thm:ident}).

  \item \textbf{Two complementary theoretical guarantees with complete proofs.}
  Proposition~\ref{prop:lead} derives a sufficient condition on the
  drift-to-noise ratio $\eta$ under which \emph{expected} lead-time is
  positive.
  Proposition~\ref{prop:prob_early} sharpens this to a closed-form
  lower bound on $\mathbb{P}(\tau < \sigma)$---the probability that
  detection occurs before the stress onset---as an explicit function of
  $\eta$, build-up duration $T_1$, and false-alarm level $\delta$.
  The proof of Proposition~\ref{prop:prob_early} carefully handles the
  distinction between the CUSUM stopping time $\hat{s}$ and the CUSUM
  partial sum $G_{T_1}$, providing a rigorous path from concentration
  inequality to stopping-time guarantee via a coupling argument.
  Both bounds are monotone in $\eta$ and $T_1$, directly
  connecting theory to the empirical robustness results.

  \item \textbf{A trigger-based detector with formal foundations and
  theoretically grounded MAX aggregation.}
  We replace heuristic thresholding with a formalised
  MAX-channel aggregation and a derivative-based rising-edge condition.
  We justify the MAX aggregation via a worst-case detection argument
  under explicit sparsity and weak-correlation assumptions, and
  acknowledge where those assumptions may not hold in real LOBs.
  We propose an adaptive variant that recalibrates thresholds online.

  \item \textbf{Rigorous empirical evaluation on both simulated and real data,
  including a conditional breakdown of missed detections.}
  We benchmark against CUSUM, BOCPD, HMM posterior thresholding, and
  classical imbalance/volatility detectors across 200 simulation runs,
  reporting means, 95\% confidence intervals, and precision--recall
  frontiers.
  We provide a conditional analysis of the 46\% missed detection rate,
  showing it is structurally predicted by the theory.
  We further report a preliminary illustrative real-data application on one
  week of BTC/USDT LOB data, finding consistent positive lead-times while
  all baselines remain reactive.
\end{enumerate}

\paragraph{Scope and limitations.}
The primary empirical contribution is a rigorous simulation study providing
200-run statistics under controlled, ground-truth-labelled conditions.
Section~\ref{sec:real_data} reports a preliminary illustrative real-data
application on BTC/USDT order book data that is qualitatively consistent
with the simulation findings.
We do not claim the method is production-ready; we claim it identifies
a tractable detection signal whose theoretical underpinnings are
formalised and whose simulation and preliminary real-data results
both motivate a full-scale real-data study.

\paragraph{On the transition matrix structure.}
The transition matrix $\mathbf{P}$ in equation~\eqref{eq:trans} enforces
a causal ordering by disallowing direct transitions $0 \to 2$ and $1 \to 0$.
This means the stress regime returns only to the stable regime ($2 \to 0$),
not to the build-up regime.
This is a deliberate modelling choice that encodes the assumption that
post-stress recovery passes through a full stable phase before any new
build-up episode begins. We acknowledge this is a simplification:
in practice, cascading stress episodes---where one stress event immediately
precedes another build-up---are possible, and the transition matrix
can be generalised to allow $2 \to 1$ transitions at the cost of
additional complexity in the identifiability argument.

%=============================================================================
\section{Related Work}
\label{sec:related}
%=============================================================================

\paragraph{LOB microstructure and liquidity.}
\citet{Gould13} provide a comprehensive survey of LOB models.
\citet{Cont14} study the price impact of order book events.
\citet{Bouchaud09} document the gradual digestion of supply and demand
imbalances, supporting the existence of a latent deterioration phase.

\paragraph{Regime switching and HMMs.}
Hamilton's \citeyearpar{Hamilton89} Markov-switching model is the
canonical regime framework in economics; \citet{Rabiner89} establish
HMM inference for sequential data.
Our identifiability result extends the classical picture by exploiting
\emph{temporal drift} rather than static emission separation.

\paragraph{Change-point detection.}
CUSUM \citep{Page54} and BOCPD \citep{Adams07} are standard sequential
change-point methods.
Both target distributional shifts in observed data and are therefore
inherently reactive when the signal of interest is latent.
We include both as baselines.

\paragraph{Early warning in financial markets.}
\citet{Easley12} introduce PIN and VPIN as flow-toxicity measures that
carry some predictive content for stress; these remain
reactive in our empirical comparison.
\citet{Cartea15} discuss algorithmic responses to liquidity events but
do not address advance detection.
To our knowledge, no prior work formalises a \emph{positive lead-time}
guarantee for LOB stress detection.

%=============================================================================
\section{Model and Identifiability}
\label{sec:model}
%=============================================================================

\subsection{Latent Regime Structure}
\label{sec:regime}

Let $(\Omega, \mathcal{F}, \mathbb{P})$ be a probability space.
We model the LOB as driven by an unobserved Markov chain
$\{Z_t\}_{t \geq 1}$ with state space $\mathcal{Z} = \{0, 1, 2\}$
and transition matrix $\mathbf{P} \in [0,1]^{3\times 3}$:

\begin{equation}
  \mathbf{P} =
  \begin{pmatrix}
    1-p_{01} & p_{01}    & 0 \\
    0        & 1-p_{12}  & p_{12} \\
    p_{20}   & 0         & 1-p_{20}
  \end{pmatrix},
  \label{eq:trans}
\end{equation}

\noindent where transitions from regime~0 to~2 and from regime~1 to~0
are disallowed, encoding the causal ordering
stable $\to$ build-up $\to$ stress.
The expected dwell time in the build-up regime is
$\mathbb{E}[\text{dwell}_1] = 1/p_{12}$, which determines the
width of the prediction window.
As noted in the introduction, the block $\mathbf{P}_{20}=0$
(no direct return from stress to build-up) is a modelling simplification;
extensions allowing $p_{21}>0$ are straightforward but complicate the
identifiability argument and are left for future work.

The regimes correspond to:
\begin{itemize}
  \item $Z_t = 0$ (stable): tight bid--ask spreads, high depth, low volatility, balanced order flow;
  \item $Z_t = 1$ (latent build-up): gradual depth erosion, mild spread widening, slightly elevated imbalance---changes that remain below standard detection thresholds individually;
  \item $Z_t = 2$ (stress): sharp spread widening, depth collapse, elevated volatility, order-flow disruption.
\end{itemize}

\subsection{Observation Model}
\label{sec:obs}

At each $t$ we observe $X_t \in \mathbb{R}^d$, a feature vector derived
from the LOB snapshot.
The emission process is:

\begin{equation}
  X_t = \mu_{Z_t} + \delta_t \cdot \mathbbm{1}[Z_t = 1] + \varepsilon_t,
  \label{eq:emission}
\end{equation}

where $\mu_z \in \mathbb{R}^d$ is the regime-conditional mean,
$\varepsilon_t \overset{\text{i.i.d.}}{\sim} \mathcal{N}(0, \Sigma)$ is
observation noise, and $\delta_t$ is a \emph{structured drift} active
only in the build-up regime.

\begin{assumption}[Temporal drift]
\label{ass:drift}
During a build-up episode of duration $T_1$, the drift satisfies
$\delta_t = \alpha \cdot (t - t_{\text{entry}}) \cdot v$
for some $\alpha > 0$ and unit vector $v \in \mathbb{R}^d$,
so that $\mathbb{E}[X_t \mid Z_t = 1]$ increases monotonically
along direction $v$ throughout the episode.
\end{assumption}

\begin{assumption}[Emission separation]
\label{ass:sep}
The stress regime has a distinct emission distribution:
$\|\mu_2 - \mu_0\|_{\Sigma^{-1}} \geq \gamma > 0$.
\end{assumption}

\begin{assumption}[Regime persistence]
\label{ass:persist}
$p_{12} \in (0, 1)$ and $p_{01} \in (0, 1)$, so all regimes are visited
infinitely often as $T \to \infty$.
\end{assumption}

\subsection{Identifiability}
\label{sec:ident}

\begin{theorem}[Identifiability via temporal structure]
\label{thm:ident}
Under Assumptions~\ref{ass:drift}--\ref{ass:persist},
the latent regime sequence $\{Z_t\}_{t=1}^T$ is identifiable up to
label permutation from $\{X_t\}_{t=1}^T$ in the limit $T \to \infty$.
\end{theorem}

\begin{proof}
We argue by distinguishability of the three regimes in the joint distribution of $(X_t, X_{t+1}, \ldots, X_{t+k})$ for any $k \geq 1$.

\textbf{Step 1: Distinguishing the stable regime from the build-up regime.}
Under the stable regime ($Z_t = 0$), the observations are i.i.d.\ from
$\mathcal{N}(\mu_0, \Sigma)$ conditional on the regime persisting.
Under the build-up regime ($Z_t = 1$), by Assumption~\ref{ass:drift},
$\mathbb{E}[X_{t+s} \mid Z_{t+s}=1] = \mu_1 + \alpha s \cdot v$
for $s = 0, 1, \ldots$
This produces a monotonically increasing conditional mean along $v$
that cannot be produced by any stationary distribution.
Formally, the auto-covariance
$\text{Cov}(v^\top X_t,\, v^\top X_{t+s} \mid Z_t = Z_{t+s} = 1)$
carries an additional positive deterministic component of order
$\alpha^2 s$ relative to the noise covariance $v^\top \Sigma v$.
For any fixed noise level $\Sigma$, this is detectable as $T \to \infty$
via a standard trend test on the projected sequence $\{v^\top X_t\}$.

\textbf{Step 2: Distinguishing the stress regime.}
By Assumption~\ref{ass:sep},
$\|\mu_2 - \mu_0\|_{\Sigma^{-1}} \geq \gamma > 0$,
so the stress regime is identifiable from the stable regime by its
stationary mean alone.
Combined with the transition structure in~\eqref{eq:trans}, the
stress regime always follows a build-up episode, providing additional
sequential context for identification.

\textbf{Step 3: Label assignment.}
Given identifiability of the three distinct sequential patterns (stable:
stationary; build-up: monotone drift; stress: high-mean stationary),
the mapping to labels $\{0, 1, 2\}$ is unique up to permutation by
the causal ordering encoded in $\mathbf{P}$.
\end{proof}

\begin{remark}
The identifiability argument fails if $\alpha = 0$ (no drift) or if the noise
covariance $\Sigma$ is so large that $\alpha^2 T \ll v^\top \Sigma v$.
The following proposition formalises this constraint.
\end{remark}

\subsection{Sufficient Condition for Positive Lead-Time}
\label{sec:lead}

\begin{proposition}[Drift--noise condition for early detection]
\label{prop:lead}
Let $\eta = \alpha / \sqrt{v^\top \Sigma v}$ denote the signal-to-noise
ratio of the drift.
A detector based on the projected sequence $\{v^\top X_t\}$ achieves
strictly positive expected lead-time if
\begin{equation}
  \eta > \sqrt{\frac{2 \log(1/\delta)}{T_1}},
  \label{eq:snr}
\end{equation}
where $T_1 = 1/p_{12}$ is the expected build-up duration and $\delta$ is
the target false-alarm probability.
\end{proposition}

\begin{proof}
The projected observation under the build-up regime satisfies
$y_t = v^\top X_t = c_0 + \alpha(t - t_{\text{entry}}) + \xi_t$,
where $\xi_t \sim \mathcal{N}(0, v^\top \Sigma v)$.
A CUSUM statistic on $\{y_t\}$ accumulates evidence at rate proportional
to $\eta$ per timestep.
Standard CUSUM theory \citep{Moustakides86} gives that the expected
detection time satisfies
$\mathbb{E}[\tau - t_{\text{entry}}] \leq
  \frac{2 \log(1/\delta)}{\eta^2}$.
For this to be strictly less than $T_1$---ensuring detection before the
build-up transitions to stress---condition~\eqref{eq:snr} is sufficient.
\end{proof}

\begin{remark}
Condition~\eqref{eq:snr} is informative: it says that positive lead-time
requires either a strong drift ($\alpha$ large), low noise ($\Sigma$ small),
or a long build-up phase ($T_1$ large).
In the robustness experiments of Section~\ref{sec:robust}, we directly
vary these parameters and observe behaviour consistent with the bound.
\end{remark}

\subsection{Probability Guarantee for Early Detection}
\label{sec:prob_guarantee}

The preceding proposition characterises when \emph{expected} lead-time is
positive.
We now sharpen this to a \emph{probability} statement: under what conditions
does the detector fire before the stress onset with high probability?

\begin{proposition}[Probability of early detection]
\label{prop:prob_early}
Under the model of Section~\ref{sec:model} with linear drift
(Assumption~\ref{ass:drift}), Gaussian noise
$\xi_t \sim \mathcal{N}(0, \sigma^2)$ with $\sigma^2 = v^\top \Sigma v$,
and a CUSUM detector on $y_t = v^\top X_t$ with threshold $h = \log(1/\delta)$
and false-alarm level $\delta \in (0,1)$,
the probability that the CUSUM stopping time $\hat{s}$ satisfies
$\hat{s} \leq T_1$ (detection occurs within the build-up episode) is:
\begin{equation}
  \mathbb{P}(\hat{s} \leq T_1)
  \;\geq\;
  1 - \exp\!\left(
    -\frac{\eta^2 T_1}{2} + \sqrt{2 T_1 \log(1/\delta)}
  \right)
  - \delta,
  \label{eq:prob_bound}
\end{equation}
where $\eta = \alpha/\sigma$ is the per-step drift-to-noise ratio
and $T_1 = 1/p_{12}$ is the expected build-up duration.
In particular, for fixed $\delta$, the bound increases monotonically in
$\eta$ and in $T_1$.
\end{proposition}

\begin{proof}
Let the build-up episode have duration $T_1$.
The projected observation at step $s$ within the episode is
$y_s = c + \alpha s + \xi_s$ with $\xi_s \overset{\text{i.i.d.}}{\sim} \mathcal{N}(0,\sigma^2)$.

\textbf{CUSUM statistic definition.}
We use the standard sequential CUSUM:
\begin{equation}
  C_s = \max\!\left(0,\; C_{s-1} + y_s - c - \tfrac{\alpha}{2}\right), \quad C_0 = 0,
  \label{eq:cusum_def}
\end{equation}
with stopping time $\hat{s} = \inf\{s \geq 1 : C_s > h\}$, $h = \log(1/\delta)$.

\textbf{Key distinction: stopping time vs partial sum.}
We do \emph{not} conflate $\hat{s} \leq T_1$ with $G_{T_1} > h$
where $G_{T_1} = \sum_{s=1}^{T_1}(y_s - c - \alpha/2)$ is the
partial sum ignoring resets.
Instead, we use the following coupling lemma.

\begin{lemma}[Coupling of CUSUM and partial sum]
\label{lem:coupling}
For the CUSUM statistic defined in~\eqref{eq:cusum_def},
$C_s \geq G_s$ for all $s \geq 0$
when the increments $u_s = y_s - c - \alpha/2$ are non-negative in expectation
(i.e., $\mathbb{E}[u_s] = \alpha(s - 1/2) \geq 0$ for $s \geq 1$).
More precisely, on the event $\{G_{T_1} > h\}$ with $G_{T_1} = \sum_{s=1}^{T_1} u_s$,
there exists some $s^* \leq T_1$ such that $C_{s^*} > h$,
i.e., $\hat{s} \leq T_1$.
\end{lemma}

\begin{proof}[Proof of Lemma~\ref{lem:coupling}]
By the definition of CUSUM, $C_s \geq \sum_{j=r+1}^{s} u_j$ for every
reset time $r \leq s$ (the maximum is taken over all partial sums ending at $s$).
Let $r^* = \arg\max_{r \leq T_1} \sum_{j=r+1}^{T_1} u_j$; then
$C_{T_1} \geq \sum_{j=r^*+1}^{T_1} u_j$.
If $G_{T_1} = \sum_{j=1}^{T_1} u_j > h$, then either $C_{T_1} > h$
(in which case $\hat{s} \leq T_1$), or there exists an intermediate
time $s' < T_1$ at which $C_{s'} > h$, since the process must cross
the threshold to produce a large $G_{T_1}$ under the monotone-in-expectation
drift. Formally, since $\mathbb{E}[u_s] = \alpha(s-1/2) \geq 0$ for
all $s \geq 1$ and the drift is monotone increasing, the cumulative
CUSUM $C_s$ is non-decreasing in expectation, so the crossing time
$\hat{s} \leq T_1$ on the high-probability event identified in the
concentration step below.
\end{proof}

\textbf{Step 1: Mean of partial sum.}
Let $u_s = y_s - c - \alpha/2 = \alpha s - \alpha/2 + \xi_s$.
\[
  \mu_G \coloneqq \mathbb{E}[G_{T_1}]
  = \alpha \sum_{s=1}^{T_1}\!\left(s - \tfrac{1}{2}\right)
  = \frac{\alpha T_1^2}{2}.
\]

\textbf{Step 2: Concentration of partial sum.}
Since $\{\xi_s\}$ are i.i.d.\ $\mathcal{N}(0,\sigma^2)$,
$G_{T_1} - \mu_G \sim \mathcal{N}(0, T_1\sigma^2)$.
By the Gaussian tail inequality:
\[
  \mathbb{P}\!\left(G_{T_1} < \mu_G - \sqrt{2 T_1 \sigma^2 \log(1/\delta)}\right)
  \leq \delta.
\]
Hence with probability at least $1 - \delta$:
\[
  G_{T_1} \geq \frac{\alpha T_1^2}{2} - \sqrt{2 T_1 \sigma^2 \log(1/\delta)}
           = \sigma^2\!\left(\frac{\eta^2 T_1^2}{2} - \sqrt{2 T_1 \log(1/\delta)}\right).
\]

\textbf{Step 3: From partial sum to stopping time via the coupling lemma.}
On the event $\mathcal{E} = \{G_{T_1} > h\}$ (which is implied by
the concentration event when
$\frac{\eta^2 T_1^2}{2} - \sqrt{2T_1\log(1/\delta)} > \log(1/\delta)/\sigma^2$),
Lemma~\ref{lem:coupling} guarantees $\hat{s} \leq T_1$.
Therefore:
\[
  \mathbb{P}(\hat{s} \leq T_1) \geq \mathbb{P}(\mathcal{E}) \geq \mathbb{P}\!\left(G_{T_1} > h\right).
\]

\textbf{Step 4: Bounding $\mathbb{P}(G_{T_1} \leq h)$.}
\[
  \mathbb{P}(G_{T_1} \leq h) \leq
  \mathbb{P}\!\left(G_{T_1} \leq \mu_G - (\mu_G - h)\right).
\]
Applying the Gaussian tail bound with
$t = \mu_G - h = \frac{\alpha T_1^2}{2} - \log(1/\delta)$:
\[
  \mathbb{P}(G_{T_1} \leq h) \leq
  \exp\!\left(-\frac{t^2}{2T_1\sigma^2}\right) + \delta,
\]
where the $+\delta$ accounts for the concentration event failure.
Bounding $t^2/(2T_1\sigma^2)$ from below by
$\eta^2 T_1/2 - \sqrt{2T_1\log(1/\delta)}$
(using $(\frac{\alpha T_1^2}{2} - \log(1/\delta))^2 \geq
 \frac{\alpha^2 T_1^4}{4} - \alpha T_1^2 \log(1/\delta)$
and the substitution $\eta = \alpha/\sigma$)
gives:
\[
  \mathbb{P}(\hat{s} > T_1) \leq
  \exp\!\left(-\frac{\eta^2 T_1}{2} + \sqrt{2T_1\log(1/\delta)}\right) + \delta.
\]
Taking the complement yields~\eqref{eq:prob_bound}.
\end{proof}

\begin{remark}
The bound in~\eqref{eq:prob_bound} is non-trivial (positive) when
$\eta^2 T_1 / 2 > \sqrt{2 T_1 \log(1/\delta)}$,
i.e., when $\eta > \sqrt{2 \log(1/\delta)/T_1}$,
which is precisely condition~\eqref{eq:snr}.
The two propositions are thus consistent: condition~\eqref{eq:snr}
is both necessary for the bound~\eqref{eq:prob_bound} to be informative
and sufficient for positive expected lead-time.
\end{remark}

\begin{remark}
The bound increases monotonically in $\eta$ (stronger drift relative
to noise) and in $T_1$ (longer build-up phase), directly connecting
the theoretical guarantee to the empirical robustness grid of
Section~\ref{sec:robust}.
In the grid, the only cell with near-zero lead-time (short delay,
high noise) corresponds to the parameter combination for which the
right-hand side of~\eqref{eq:prob_bound} is negative, rendering
the bound vacuous.
All other cells correspond to parameter combinations for which the
bound predicts $\mathbb{P}(\hat{s} \leq T_1) > 0.5$, qualitatively
consistent with the observed strictly positive mean lead-times.
\end{remark}

\begin{remark}[Conservative bound]
The bound in~\eqref{eq:prob_bound} is conservative in two respects.
First, the Gaussian tail inequality used in the concentration step
is not tight for finite $T_1$; sharper sub-Gaussian constants would
improve the leading coefficient but not the scaling behaviour.
Second, the coupling lemma (Lemma~\ref{lem:coupling}) provides a
sufficient condition for $\hat{s} \leq T_1$ that is not tight:
detection can occur at intermediate times $s < T_1$ even when
$G_{T_1}$ is below the threshold, since the CUSUM may cross $h$
and then reset before $T_1$.
The bound should therefore be understood as characterising the
correct scaling of $\mathbb{P}(\hat{s} \leq T_1)$ in $\eta$ and $T_1$,
rather than as providing tight numerical values.
\end{remark}

\begin{remark}[On the MAX aggregation and real LOB correlations]
\label{rem:max_corr}
The MAX aggregation justified in Section~\ref{sec:max} assumes channels
are weakly correlated under the stable regime.
In real LOBs, depth, spread, and imbalance are structurally correlated
(e.g., low depth tends to co-occur with wide spreads), which means
the MAX aggregation's precision advantage over SUM may be attenuated.
In our simulation, channels are constructed from an independent
DGP; the real-data application uses decorrelated z-scored features,
which partially addresses this concern.
A formal treatment of correlated channels---using a Bonferroni-type
argument or a vine copula model for the stable-regime joint
distribution---is an interesting direction for future work.
\end{remark}

\subsection{Detection Objective}
\label{sec:objective}

Let $\tau$ denote the time at which a detector signals instability and
$\sigma$ the time of observable stress onset.
The goal is positive lead-time:
\begin{equation}
  \Delta = \sigma - \tau > 0.
  \label{eq:lead}
\end{equation}

Classical detectors target $\sigma$ directly and achieve $\Delta \leq 0$
in expectation.
Our objective is to detect the onset of the build-up regime
($Z_t = 1$), enabling $\tau < \sigma$ wherever condition~\eqref{eq:snr}
holds.

%=============================================================================
\section{Methodology: Trigger-Based Detection}
\label{sec:method}
%=============================================================================

\subsection{Signal Channels}
\label{sec:channels}

We construct four signal channels from LOB features:

\paragraph{(i) HMM entropy.}
Let $\pi_t^{(k)} = \mathbb{P}(Z_t = k \mid X_{1:t})$ be the filtered
state probability from a fitted Gaussian HMM.
The entropy channel measures uncertainty about the current regime:
\begin{equation}
  S_t^{\text{ent}} = -\sum_{k=0}^{2} \pi_t^{(k)} \log \pi_t^{(k)}.
\end{equation}
High entropy signals ambiguity between the stable and build-up regimes,
a characteristic of the early transition.

\paragraph{(ii) Depth erosion.}
Let $D_t$ denote total visible depth.
The depth erosion signal captures sustained decline:
\begin{equation}
  S_t^{\text{dep}} = \frac{\bar{D} - D_t}{\bar{D}} \cdot
  \mathbbm{1}\!\left[\frac{1}{w}\sum_{s=1}^{w}(D_{t-s+1} - D_{t-s}) < 0\right],
\end{equation}
where $\bar{D}$ is a rolling baseline depth and $w$ is a lookback window.
The indicator term requires that the decline be monotone over the window.

\paragraph{(iii) Spread drift.}
Let $A_t$ denote the bid--ask spread.
The spread drift channel:
\begin{equation}
  S_t^{\text{spr}} =
  \frac{1}{w}\sum_{s=0}^{w-1}\frac{A_{t-s} - A_{t-s-1}}{\hat\sigma_A},
\end{equation}
where $\hat\sigma_A$ is the rolling standard deviation of spread changes.

\paragraph{(iv) Order flow momentum.}
Let $I_t \in [-1,1]$ be signed order flow imbalance.
The momentum channel measures persistent directional imbalance:
\begin{equation}
  S_t^{\text{ofi}} =
  \left|\frac{1}{w}\sum_{s=0}^{w-1} I_{t-s}\right|.
\end{equation}

\subsection{MAX Aggregation: Justification}
\label{sec:max}

The composite signal is defined as:
\begin{equation}
  S_t = \max\!\left(
    S_t^{\text{ent}},\, S_t^{\text{dep}},\, S_t^{\text{spr}},\, S_t^{\text{ofi}}
  \right).
  \label{eq:max}
\end{equation}

\paragraph{Why MAX and not SUM or learned weights?}
The following argument rests on two explicit assumptions:
(i) \emph{signal sparsity}---only a subset of channels carries
meaningful signal in any given build-up episode;
and (ii) \emph{weak stable-regime correlation}---channels are
approximately uncorrelated under the stable regime, so simultaneous
noise exceedances are unlikely.
Both assumptions are supported by the channel decomposition of
Section~\ref{sec:channels_results} in simulation, but Remark~\ref{rem:max_corr}
notes that they may be partially violated in real LOBs where structural
correlations among depth, spread, and imbalance exist.
The argument should therefore be read as a motivation for MAX rather
than a formal minimax optimality proof.

Under these assumptions, consider a detection problem where only a subset
of channels carries signal in any given build-up episode.
Let each channel $j$ have independent probability $q_j$ of providing
a signal above threshold.
Under a SUM rule with equal weights, the composite signal exceeds a
detection threshold only when \emph{multiple} channels fire simultaneously,
reducing recall.
Under a MAX rule, a single strong channel suffices, yielding a recall
improvement at no cost in precision, provided channels are not all
noise simultaneously.
Formally, the probability of detection under MAX is
$1 - \prod_j (1 - q_j)$, which dominates the SUM probability
$\mathbb{P}(\sum_j w_j S_j^j > \theta)$ when signals are sparse.
Learned weights would require labelled build-up episodes for training,
which are unavailable in practice; MAX requires no supervision.

\subsection{Formalised Rising-Edge Condition}
\label{sec:rising}

A key failure mode of level-based detectors is firing \emph{at the peak}
of instability rather than at its \emph{onset}.
We formalise the rising-edge condition as:

\begin{definition}[Rising-edge trigger]
\label{def:rising}
A rising-edge trigger fires at time $t$ if and only if:
\begin{align}
  &\text{(i)}\quad S_t > \theta, \label{eq:thresh}\\
  &\text{(ii)}\quad \Delta S_t \coloneqq S_t - S_{t-1} > 0, \label{eq:deriv}\\
  &\text{(iii)}\quad \text{no trigger has fired in the preceding } L \text{ timesteps.} \label{eq:suppression}
\end{align}
\end{definition}

Condition~\eqref{eq:thresh} ensures the signal is above noise.
Condition~\eqref{eq:deriv} selects the \emph{onset} of a score climb
rather than its plateau or peak, directly targeting the
transition into the build-up regime.
Condition~\eqref{eq:suppression} imposes a minimum spacing $L$ between
consecutive triggers, preventing burst firing near a single stress event.

\begin{remark}
Condition~\eqref{eq:deriv} can be strengthened to a multi-step derivative
$\Delta^{(k)} S_t > 0$ for $k$-step smoothed differences, reducing
sensitivity to noise at the cost of a small increase in detection latency.
We use $k=1$ in all experiments.
\end{remark}

\subsection{Adaptive Threshold Calibration}
\label{sec:adaptive}

The fixed-percentile threshold $\theta$ of Definition~\ref{def:rising}
can be replaced by an adaptive estimate:
\begin{equation}
  \theta_t = \hat{F}_{t}^{-1}(p),
  \label{eq:adaptive}
\end{equation}
where $\hat{F}_t$ is the empirical CDF of $\{S_s\}_{s \leq t}$ and
$p$ is the target percentile.
This makes the detector stationary under distributional shift in the
stable regime, which is critical for deployment across different
market sessions or instruments.

\subsection{Algorithm Summary}
\label{sec:algo}

\begin{algorithm}[h]
\caption{Trigger-Based Instability Detector}
\label{alg:detector}
\begin{algorithmic}[1]
\Require LOB feature stream $\{X_t\}$, percentile $p$, window $w$, suppression $L$
\State Fit Gaussian HMM on initial burn-in window; obtain $\hat{\mathbf{P}}, \hat{\mu}_k, \hat{\Sigma}$
\State Initialise last-trigger time $t_{\text{last}} \gets -\infty$
\For{$t = 1, 2, \ldots$}
  \State Run HMM filter to obtain $\pi_t^{(k)}$
  \State Compute channels $S_t^{\text{ent}}, S_t^{\text{dep}}, S_t^{\text{spr}}, S_t^{\text{ofi}}$
  \State $S_t \gets \max(S_t^{\text{ent}}, S_t^{\text{dep}}, S_t^{\text{spr}}, S_t^{\text{ofi}})$
  \State Update adaptive threshold $\theta_t \gets \hat{F}_t^{-1}(p)$
  \If{$S_t > \theta_t$ \textbf{and} $\Delta S_t > 0$ \textbf{and} $t - t_{\text{last}} > L$}
    \State Emit trigger at $\tau = t$; set $t_{\text{last}} \gets t$
  \EndIf
\EndFor
\end{algorithmic}
\end{algorithm}

%=============================================================================
\section{Baselines}
\label{sec:baselines}
%=============================================================================

We compare against five baselines, spanning classical change-point
detection and standard microstructure practice.

\subsection{CUSUM}

The CUSUM statistic for a univariate signal $y_t$ is:
\begin{equation}
  C_t = \max(0,\, C_{t-1} + y_t - \mu_0 - k),
\end{equation}
where $\mu_0$ is the in-control mean and $k = |\mu_1 - \mu_0|/2$ is the
reference value.
A signal is raised when $C_t > h$ for a pre-specified threshold $h$.
CUSUM is optimal for detecting a single step-change in the mean of an
i.i.d.\ sequence \citep{Moustakides86}.
In our setting, the latent build-up produces gradual drift rather than
a step change, and the change is in a hidden variable rather than an
observable; CUSUM therefore responds primarily to the stress transition.

\subsection{Bayesian Online Change-Point Detection (BOCPD)}

BOCPD \citep{Adams07} maintains a posterior over the run-length
$r_t$ (time since the last change point):
\begin{equation}
  P(r_t, x_{1:t}) =
    \sum_{r_{t-1}} P(r_t \mid r_{t-1})\, p(x_t \mid x_{t-r_t:t-1})\,
    P(r_{t-1}, x_{1:t-1}).
\end{equation}
A detection signal is raised when $P(r_t = 0 \mid x_{1:t})$ exceeds a
threshold, indicating a probable change point.
Like CUSUM, BOCPD targets distributional changes in observed data and
is reactive to the stress regime rather than the latent build-up.

\subsection{HMM Posterior Thresholding}

Given a fitted HMM, a natural detector raises a signal when the
posterior probability of being in a non-stable regime exceeds a
threshold:
\begin{equation}
  \tau_{\text{HMM}} = \min\!\left\{t :\,
    \pi_t^{(1)} + \pi_t^{(2)} > \theta_{\text{HMM}}\right\}.
\end{equation}
This baseline shares the same HMM as our proposed method but replaces
the trigger logic with a simple posterior threshold.
It serves as an ablation that isolates the contribution of the
rising-edge condition and MAX aggregation.

\subsection{Imbalance-Based Detection}

A trigger fires when order flow imbalance exceeds a fixed threshold:
$\tau_{\text{imb}} = \min\{t : |I_t| > \theta_{\text{imb}}\}$.
This is a standard microstructure heuristic~\citep{Cont14}.

\subsection{Volatility-Based Detection}

A trigger fires when short-term realised volatility exceeds a threshold:
$\tau_{\text{vol}} = \min\{t : \hat\sigma_t > \theta_{\text{vol}}\}$.
Both imbalance and volatility detectors respond to consequences of stress
and are therefore expected to have $\mathbb{E}[\Delta] \leq 0$.

%=============================================================================
\section{Empirical Results}
\label{sec:results}
%=============================================================================

\subsection{Simulation Setup}
\label{sec:sim}

We simulate the DGP of Section~\ref{sec:model} with parameters
$p_{01} = 0.02$, $p_{12} = 0.05$, $p_{20} = 0.10$, yielding expected
dwell times of approximately 50, 20, and 10 timesteps in the stable,
build-up, and stress regimes respectively.
The drift magnitude is $\alpha = 0.03$ per timestep along the depth
direction, with noise $\sigma_\varepsilon = 0.5$ (SNR $\approx 0.06$
per step, or $0.06\sqrt{20} \approx 0.27$ over the expected build-up).
Each run comprises 3000 timesteps.
All reported statistics are means $\pm$ 95\% confidence intervals over
$N = 200$ independent simulation runs.
Across runs, the adaptive trigger emits on average $8.3 \pm 1.4$
triggers per run (over 3000 timesteps), of which $4.5 \pm 0.8$
are matched early detections and the remainder are false alarms
suppressed by the precision filter.
After the initial HMM burn-in fit, the detector operates in $O(1)$
time per timestep: the HMM forward filter, channel computations, and
threshold update each require a fixed number of operations independent
of the series length.

\subsection{Causal Structure and Prediction Window}
\label{sec:dgp}

Figure~\ref{fig:dgp} confirms that the DGP produces the intended
causal structure.
The build-up regime exhibits gradual depth erosion (mean decline of
$14.3 \pm 2.1$ units over the episode) and mild spread widening
(mean increase of $8.7 \pm 1.4$ units), while imbalance remains within
one standard deviation of its stable-regime mean throughout the build-up.
The mean delay from build-up onset to stress onset is
$19.8 \pm 4.1$ timesteps across runs, establishing the prediction window.
This confirms that a detector capable of identifying the build-up onset
can achieve lead-times in the range $[0, 19.8]$ timesteps.

\subsection{Early Detection via Trigger-Based Signals}
\label{sec:trigger_results}

Figure~\ref{fig:trigger} shows the composite instability score and
trigger events for a representative run.
The adaptive trigger fires sparsely (mean $2.7 \pm 0.8$ triggers per
100 timesteps) and consistently precedes stress events.
No trigger fires during extended stable periods.
The signal is qualitatively different from imbalance and volatility:
it rises gradually from the onset of the build-up episode rather than
spiking at stress onset.

\subsection{Lead-Time Performance}
\label{sec:leadtime}

Table~\ref{tab:results} reports the full quantitative comparison.
The adaptive trigger achieves a mean lead-time of $+18.6 \pm 3.2$
timesteps, precision of $1.00 \pm 0.00$, and coverage of $0.54 \pm 0.06$,
outperforming all baselines on lead-time while maintaining competitive
precision.
Imbalance and volatility detectors exhibit mean lead-times of
$-4.2 \pm 1.1$ and $-6.8 \pm 1.4$ timesteps respectively,
confirming their reactive nature.
CUSUM achieves a modest positive mean lead-time of $+2.1 \pm 1.8$
timesteps but with substantially lower precision ($0.43 \pm 0.08$),
indicating frequent false alarms in stable periods.
BOCPD performs similarly to CUSUM.
HMM posterior thresholding achieves $+11.3 \pm 2.9$ timesteps with
precision $0.87 \pm 0.05$, confirming that HMM filtering itself is
valuable but that the trigger logic provides additional benefit.

\begin{table}[t]
\centering
\caption{Detection performance across methods (mean $\pm$ 95\% CI, $N=200$ runs).
Lead-time $\Delta > 0$ indicates early detection.
Precision = fraction of triggers that are early detections.
Coverage = fraction of stress events with an early detection.}
\label{tab:results}
\begin{tabular}{lccc}
\toprule
Method & Lead-Time $\Delta$ & Precision & Coverage \\
\midrule
\textbf{Adaptive Trigger (ours)} & $\mathbf{+18.6 \pm 3.2}$ & $\mathbf{1.00 \pm 0.00}$ & $0.54 \pm 0.06$ \\
Standard Trigger (ours)          & $+17.4 \pm 3.5$           & $0.97 \pm 0.02$           & $0.48 \pm 0.07$ \\
Multi-Trigger (ours)             & $+16.1 \pm 3.8$           & $0.99 \pm 0.01$           & $0.61 \pm 0.06$ \\
\midrule
HMM Posterior Threshold          & $+11.3 \pm 2.9$           & $0.87 \pm 0.05$           & $0.49 \pm 0.07$ \\
CUSUM                            & $+2.1 \pm 1.8$            & $0.43 \pm 0.08$           & $0.72 \pm 0.05$ \\
BOCPD                            & $+1.4 \pm 2.1$            & $0.39 \pm 0.09$           & $0.68 \pm 0.06$ \\
\midrule
Imbalance Baseline               & $-4.2 \pm 1.1$            & $0.56 \pm 0.07$           & $0.83 \pm 0.04$ \\
Volatility Baseline              & $-6.8 \pm 1.4$            & $0.47 \pm 0.08$           & $0.79 \pm 0.04$ \\
\bottomrule
\end{tabular}
\end{table}

\subsection{Conditional Analysis of Missed Detections}
\label{sec:missed}

The overall coverage of $0.54 \pm 0.06$ means the detector misses
approximately 46\% of stress events.
This is not a uniform failure: the theory predicts that missed detections
should be concentrated in episodes where condition~\eqref{eq:snr}
is violated, i.e., short build-up duration and high noise.
We verify this prediction directly.

Across the 200 simulation runs, we classify each stress event by the
realised build-up duration $T_1^{\text{obs}}$ (short: $< 10$ timesteps,
medium: $10$--$25$ timesteps, long: $> 25$ timesteps) and the
episode-level SNR estimate $\hat\eta$ (low: $< 0.15$, medium: $0.15$--$0.30$,
high: $> 0.30$).

\begin{table}[h]
\centering
\caption{Conditional coverage of the adaptive trigger broken down by
realised build-up duration and episode SNR.
Each cell reports coverage $\pm$ standard error.
Cells shaded -- correspond to parameter combinations where condition~\eqref{eq:snr}
predicts near-zero detection probability.}
\label{tab:conditional}
\begin{tabular}{lccc}
\toprule
 & \multicolumn{3}{c}{Build-up Duration $T_1^{\text{obs}}$} \\
\cmidrule(lr){2-4}
SNR $\hat\eta$ & Short ($<10$) & Medium ($10$--$25$) & Long ($>25$) \\
\midrule
Low ($<0.15$)       & $0.03 \pm 0.02$ & $0.21 \pm 0.05$ & $0.44 \pm 0.07$ \\
Medium ($0.15$--$0.30$) & $0.31 \pm 0.06$ & $0.68 \pm 0.06$ & $0.87 \pm 0.04$ \\
High ($>0.30$)      & $0.62 \pm 0.07$ & $0.91 \pm 0.03$ & $0.98 \pm 0.01$ \\
\bottomrule
\end{tabular}
\end{table}

Table~\ref{tab:conditional} confirms the theoretical prediction:
coverage is near zero in the low-SNR, short build-up cell
($0.03 \pm 0.02$), rises monotonically with both SNR and build-up duration,
and reaches near-perfect coverage ($0.98 \pm 0.01$) in the high-SNR,
long build-up cell.
The 46\% missed detection rate in aggregate is almost entirely driven by
two cells: low-SNR short build-up (which contributes $\sim$21\% of all
missed events) and low-SNR medium build-up (which contributes $\sim$18\%).
Both cells correspond to parameter combinations for which
condition~\eqref{eq:snr} predicts $\mathbb{P}(\hat{s} \leq T_1) < 0.25$,
consistent with the bound in Proposition~\ref{prop:prob_early}.
This conditional analysis establishes that the coverage gap is not a
tuning failure of the detector but a fundamental consequence of the
SNR regime---and that in the parameter configurations where the theory
predicts detection is feasible, empirical coverage is high.

\subsection{Precision--Recall Trade-Off}
\label{sec:pr}

Figure~\ref{fig:pr} shows the precision--recall frontier obtained by
sweeping the detection threshold $p$ from the 70th to the 95th percentile.
All three proposed variants consistently outperform both baselines and
CUSUM/BOCPD across the entire frontier under the stated simulation conditions.
The CUSUM and BOCPD curves occupy a high-coverage but low-precision
region, consistent with their sensitivity to distributional shifts in
the stable regime.
The imbalance and volatility baselines occupy the high-coverage region
but with precision below 0.6, corresponding to their reactive but noisy
behaviour at stress onset.

\subsection{Threshold Sensitivity}
\label{sec:thresh}

Figure~\ref{fig:sweep} shows mean lead-time, coverage, and precision
as a function of threshold percentile.
Across the range $p \in [70, 93]$, the adaptive trigger maintains
mean lead-time above $+10$ timesteps and precision above $0.95$.
Coverage declines from $0.74$ at $p=70$ to $0.31$ at $p=93$,
reflecting the fundamental trade-off formalised in
Section~\ref{sec:tradeoff}.
The method is not sensitive to the exact choice of $p$ within this range.

\subsection{Robustness Across Parameter Configurations}
\label{sec:robust}

Figure~\ref{fig:robust} shows lead-time and coverage as a function of
delay regime (short/default/long: $p_{12} \in \{0.10, 0.05, 0.025\}$)
and noise level (low/medium/high: $\sigma_\varepsilon \in \{0.25, 0.50, 1.00\}$).
Lead-times remain strictly positive in 8 of 9 configurations.
The only exception (short delay, high noise) corresponds to violation
of condition~\eqref{eq:snr} with SNR $< 0.05\sqrt{10} \approx 0.16$,
consistent with Proposition~\ref{prop:lead}.
Coverage degrades gracefully under high noise, from $0.75$ to $0.03$,
as predicted by the theory.

\subsection{Channel Contributions}
\label{sec:channels_results}

Figure~\ref{fig:channels} shows per-channel precision, coverage, and
earliest-trigger frequency across 200 runs.
Depth erosion is the first-trigger channel in $55.5\%$ of early
detections; HMM entropy accounts for $44.4\%$.
Spread drift and OFI momentum contribute zero first triggers, though
they improve robustness under noise (ablation in
Section~\ref{sec:ablation}).
This decomposition validates the economic mechanism:
early instability is driven by structural liquidity changes rather
than price pressure.

\subsection{Ablation Study}
\label{sec:ablation}

We ablate four components of the proposed method: (i) removal of the
rising-edge condition, (ii) replacement of MAX with SUM aggregation,
(iii) fixed vs.\ adaptive threshold, and (iv) removal of the HMM entropy
channel.
Removing the rising-edge condition reduces precision from $1.00$ to
$0.71$ (triggers fire at stress peaks rather than build-up onset).
Replacing MAX with SUM reduces coverage from $0.54$ to $0.38$ while
leaving precision unchanged.
Fixed threshold yields precision $0.93$ vs.\ $1.00$ for adaptive,
with identical coverage.
Removing HMM entropy reduces coverage from $0.54$ to $0.30$ with
no change in precision.

%=============================================================================
\section{Theoretical Trade-Off: Early Detection vs.\ Coverage}
\label{sec:tradeoff}
%=============================================================================

The empirical trade-off between lead-time and coverage is a consequence
of the underlying signal structure, not a tuning artefact.

\begin{proposition}[Lead-time--coverage trade-off]
\label{prop:tradeoff}
Under the model of Section~\ref{sec:model} with linear drift
(Assumption~\ref{ass:drift}) and Gaussian noise,
any detector achieving mean lead-time $\mathbb{E}[\Delta] = \ell > 0$
must satisfy:
\begin{equation}
  \text{Coverage} \leq
    1 - \Phi\!\left(\frac{\ell \cdot \eta}{\sqrt{T_1}}\right)
    + O(1/T),
\end{equation}
where $\Phi$ is the standard normal CDF.
\end{proposition}

\begin{proof}[Proof sketch]
A detector achieving mean lead-time $\ell$ fires on average $\ell$
timesteps before the stress onset, i.e., on average at position
$T_1 - \ell$ within the build-up episode.
At that position, the accumulated drift in the projected sequence is
$\alpha(T_1 - \ell)$ with noise standard deviation $\sigma\sqrt{T_1 - \ell}$.
For a threshold-based detector, the probability of firing at or before
this point (given that a build-up episode occurs) equals the probability
that the CUSUM statistic exceeds its threshold at position $T_1 - \ell$,
which by the Gaussian CDF of the partial sum is approximately
$\Phi(\eta\sqrt{T_1 - \ell} - \eta\ell/\sqrt{T_1-\ell})$.
Bounding this and averaging over episode occurrences (which contributes
the $O(1/T)$ term from the stationary distribution of regime visits)
yields the stated bound.
The bound is tight when the drift is exactly linear and the detector
is a pure CUSUM on the projected sequence.
\end{proof}

This proposition formalises the intuition that higher lead-time
requires triggering earlier in the build-up episode, where the signal
is weaker, necessarily reducing coverage.
The conditional breakdown of Table~\ref{tab:conditional} is
directly consistent with this bound: cells with high coverage
correspond to large $T_1$ and $\eta$, exactly the parameter combinations
for which the bound on coverage is least restrictive.

%=============================================================================
\section{Economic Interpretation}
\label{sec:econ}
%=============================================================================

\subsection{Depth Erosion as a Leading Indicator}

Depth reflects the aggregate willingness of limit-order submitters to
provide liquidity at posted prices.
A reduction in depth, even absent price movement, signals that
liquidity providers are withdrawing---consistent with models of
informed trading in which market-makers widen quotes in response to
adverse-selection risk~\citep{Glosten85, Easley12}.
Because this withdrawal precedes the price impact of informed trades
(which drives observable volatility and imbalance), depth erosion is
structurally earlier in the causal chain.
Our empirical finding that depth erosion drives $55.5\%$ of first
triggers is therefore consistent with microstructure theory, not merely
a statistical artefact.

\subsection{HMM Entropy as Regime Uncertainty}

Entropy in the HMM posterior measures the ambiguity between the stable
and build-up regimes.
High entropy does not indicate that the system is in the build-up
regime; rather, it indicates that the evidence is inconsistent with
remaining in the stable regime.
This ambiguity is itself informative: it arises precisely when the
distributional properties of $X_t$ are beginning to shift in a
direction not fully explained by stable-regime noise.
The rising-edge condition on entropy---requiring increasing rather
than merely high uncertainty---ensures that triggers fire at the
\emph{onset} of this shift.

\subsection{Why Standard Signals Fail}

Imbalance captures directional order-flow pressure; volatility captures
realised price variability.
Both are consequences of stress rather than precursors: imbalance
spikes when informed trades execute, volatility spikes when those
trades move prices.
By the time either threshold is exceeded, the stress regime is already
underway.
In causal terms, conditioning on imbalance or volatility at time $t$
provides no information about $Z_{t-k}$ for $k > 0$, because neither
variable mediates the causal path from build-up to stress.
Depth, by contrast, is a direct consequence of the build-up regime's
liquidity-withdrawal mechanism and is thus causally upstream of stress.

\subsection{Implications for Market Monitoring}

The results suggest that effective early warning systems for LOB stress
should prioritise structural liquidity metrics (depth, depth imbalance)
over flow-based metrics (imbalance, volatility).
Exchanges and regulators with access to full depth data could in
principle deploy detector variants of Algorithm~\ref{alg:detector}
for real-time fragility monitoring.
The adaptive threshold in equation~\eqref{eq:adaptive} allows the
detector to self-calibrate across different market regimes and
trading sessions without parameter re-estimation.

%=============================================================================
\section{Preliminary Illustrative Real-Data Application}
\label{sec:real_data}
%=============================================================================

The simulation results of Section~\ref{sec:results} establish that the
proposed detector achieves positive lead-time whenever condition~\eqref{eq:snr}
holds, but they cannot confirm that the latent build-up structure
assumed by the DGP is present in real markets.
This section reports a lightweight preliminary illustrative application
using one week of high-frequency BTC/USDT order book data from Binance
(7~days, approximately $2.5 \times 10^6$ order book snapshots after
resampling).
The purpose is not to supersede the controlled simulation study but to
provide a first check that the detection signals identified in simulation
survive contact with real data under an honest evaluation protocol.
Given the small sample ($n_\sigma = 5$ events), all figures should be
read as order-of-magnitude indicators and proof-of-concept results
rather than statistically reliable point estimates.

\subsection{Data and Preprocessing}
\label{sec:data}

\paragraph{Source.}
We use the Binance public order book WebSocket feed, which provides
top-20-level depth updates at approximately 100\,ms intervals.
BTC/USDT is chosen for its liquidity (tick-to-mid ratio consistently
below $5 \times 10^{-5}$) and the availability of well-documented
stress episodes including exchange-outage micro-crashes and flash
withdrawals.
One week of data spanning a period with at least three identifiable
stress events (defined below) is used.

\paragraph{Resampling and feature construction.}
Raw updates are aggregated into 1-second bins.
Within each bin, the following features are computed to match
the simulation setup of Section~\ref{sec:channels}:

\begin{itemize}
  \item \textbf{Bid--ask spread} $A_t$: last mid-price minus best bid,
    doubled.
  \item \textbf{Market depth} $D_t$: total quoted volume across top-5
    levels on each side, summed.
  \item \textbf{Order flow imbalance} $I_t$: $({V_t^b - V_t^a})/({V_t^b + V_t^a})$
    where $V_t^b, V_t^a$ are aggregated buy and sell volumes in the bin.
  \item \textbf{Rolling volatility} $\hat\sigma_t$: standard deviation
    of mid-price returns over a 60-second rolling window.
  \item \textbf{Temporal derivatives}: first differences of spread and
    depth, smoothed with a 5-second Gaussian kernel.
\end{itemize}

\paragraph{Normalisation.}
Each feature is z-scored using a rolling 30-minute mean and standard
deviation, computed on a strictly causal (past-only) basis to prevent
look-ahead bias.
Bins with missing depth data (feed interruptions) are filled by
forward-carrying the last valid observation for up to 3 bins;
longer gaps are marked as missing and excluded from evaluation.
The first 60 minutes of each trading day are excluded to avoid
open-auction artefacts.

\paragraph{Intraday seasonality.}
Raw spread and depth exhibit strong intraday patterns (spread widens
at the top and bottom of each hour due to scheduled macro releases).
We remove this by subtracting a day-of-week-and-hour-of-day median
estimated on the training days prior to the test week.
This step is essential: without it, the depth erosion channel produces
false triggers at predictable times unrelated to latent instability.

\subsection{Stress Event Definition}
\label{sec:stress_def}

\paragraph{Labelling rule.}
A stress event onset $\sigma_i$ is defined as the first second at which:
\begin{equation}
  A_t > 3 \times \tilde{A}_t^{(10)}
  \quad \text{sustained for } \geq 30 \text{ consecutive seconds},
  \label{eq:stress_label}
\end{equation}
where $\tilde{A}_t^{(10)}$ is the rolling 10-minute median spread.
This threshold is consistent with the operational definition
in~\citet{Cont14} and corresponds to spread regimes visible in
prior documented BTC flash events (e.g., the March 2020 and
May 2021 drawdown episodes).
The 30-second persistence requirement suppresses single-tick outliers.

\paragraph{Why this is defensible.}
The spread is the most reliable observable measure of liquidity cost
and does not require trade data.
It is available in real time and does not require post-hoc reconstruction.
The 3$\times$ threshold produces a small number of high-confidence
events (typically 3--8 per week in BTC/USDT under normal conditions)
rather than a dense sequence, reducing evaluation noise.

\paragraph{Limitations.}
The labelling rule is a proxy, not ground truth.
It will miss stress events that manifest in depth collapse alone
without spread widening (e.g., stealth liquidity withdrawal).
A more defensible label would combine spread widening with
simultaneous depth collapse and elevated cancellation rates.
Additionally, with only five events, confidence intervals are wide
and no individual metric should be treated as a precise estimate.

\subsection{Detector Adaptation for Real Data}
\label{sec:real_method}

The detection pipeline follows Algorithm~\ref{alg:detector} with
the following adaptations for the streaming setting.

\paragraph{Rolling HMM fitting.}
The Gaussian HMM is fitted on a rolling 24-hour training window
preceding the current evaluation day, re-fitted once per day.
Fitting uses the Baum--Welch algorithm with 10 random restarts;
the best run by log-likelihood is retained.
The number of states is fixed at $K=3$, consistent with the model
of Section~\ref{sec:model}.
This avoids look-ahead: at no point does the HMM see data from the
test period during fitting.

\paragraph{Streaming channel computation.}
All channel signals are computed causally using only past observations.
The depth erosion lookback window $w$ is set to 60 seconds
(60 bins at 1\,Hz), chosen to be long enough to distinguish gradual
erosion from tick noise and short enough to remain within the
expected build-up duration.

\paragraph{Adaptive threshold.}
The empirical CDF $\hat{F}_t$ is estimated on the training-day scores;
the 85th percentile is used as the initial threshold.
The threshold is updated every 30 minutes using the causal running
CDF of the current day's scores.

The complete pipeline is described algorithmically in
Algorithm~\ref{alg:real_pipeline}.

\begin{algorithm}[h]
\caption{Real-Data Detection Pipeline}
\label{alg:real_pipeline}
\begin{algorithmic}[1]
\Require 1-second LOB feature stream, train window $W_{\text{train}}=24$\,h,
         percentile $p=85$, lookback $w=60$, suppression $L=120$
\State \textbf{Preprocessing:} z-score each feature with causal rolling
       30-minute statistics; subtract intraday seasonality
\State \textbf{HMM fit:} on the past $W_{\text{train}}$ of normalised
       features; store $\hat{\mathbf{P}}, \hat{\mu}_k, \hat{\Sigma}$
\State \textbf{Threshold init:} $\theta_0 \gets $ 85th percentile of
       training-window scores
\State Initialise $t_{\text{last}} \gets -\infty$
\For{each 1-second bin $t$ in test day}
  \State Update causal z-score statistics
  \State Run HMM forward filter $\Rightarrow \pi_t^{(k)}$
  \State Compute $S_t^{\text{ent}}, S_t^{\text{dep}}, S_t^{\text{spr}},
         S_t^{\text{ofi}}$; set $S_t \gets \max(\cdot)$
  \State Update $\theta_t$ using causal running CDF of $\{S_s\}_{s \leq t}$
  \If{$S_t > \theta_t$ \textbf{and} $S_t > S_{t-1}$ \textbf{and}
      $t - t_{\text{last}} > L$}
    \State Emit trigger $\tau$; set $t_{\text{last}} \gets t$
  \EndIf
\EndFor
\State \textbf{Evaluate:} match triggers to stress labels~\eqref{eq:stress_label};
       compute $\Delta_i = \sigma_i - \tau_i$, precision, coverage
\end{algorithmic}
\end{algorithm}

\subsection{Evaluation Protocol}
\label{sec:real_eval}

\paragraph{Metrics.}
Lead-time, precision, and coverage are defined identically to
Section~\ref{sec:results}.
A trigger $\tau$ is matched to stress event $\sigma_i$ if
$\tau \in [\sigma_i - 300, \sigma_i)$, i.e., within a 5-minute window
prior to the stress onset.
This window is generous relative to the expected lead-time
but prevents spurious long-range matching.
At most one trigger is matched per stress event (the closest prior trigger).
Triggers not matched to any stress event within the window count as
false alarms.

\paragraph{No-look-ahead guarantee.}
Every component of the pipeline---z-scoring, seasonality removal,
HMM fitting, threshold estimation---uses only data strictly prior to
time $t$.
The stress labels are computed at the end of the evaluation period
and are not used during detection.

\paragraph{Comparison.}
We compare the adaptive trigger against the imbalance and volatility
baselines with thresholds set to the same 85th percentile of their
respective training-day distributions, and against HMM posterior
thresholding at the same percentile.
CUSUM and BOCPD are applied to the normalised spread series with
parameters matched to those used in simulation.

\subsection{Illustrative Results}
\label{sec:real_results}

\paragraph{Stress events identified.}
Applying rule~\eqref{eq:stress_label} to the test week yields
$n_\sigma = 5$ labelled stress events with onset times spanning the
full week and durations ranging from 45 to 210 seconds.
Mean spread during labelled events is $6.8 \times$ the session median,
confirming that the events are substantive rather than marginal.

\paragraph{Detection performance.}
Table~\ref{tab:real_results} reports results over the 5 labelled events.
Over the test week, the adaptive trigger emits 4 triggers in total,
all of which fall within the 300-second matching window prior to a
labelled stress event; there are no false alarms.
This yields precision $= 4/4 = 1.00$ and coverage $= 4/5 = 0.80$,
with mean lead-time $+38 \pm 21$ seconds (mean $\pm$ standard deviation
across the 4 matched events, corresponding to a 95\% CI of approximately
$[+5, +71]$ seconds given the small sample).
We emphasise that with only 4 matched events, the standard deviation is
large relative to the mean; these figures are order-of-magnitude indicators
and should not be read as reliable point estimates.
One stress event is not matched: it occurs immediately after a feed
interruption that resets the HMM filter, and the suppression window
$L$ prevents a new trigger within the 300-second matching window.
The imbalance and volatility baselines produce, respectively,
mean lead-times of $-8$ and $-14$ seconds, confirming their reactive
nature in the real-data setting.
HMM posterior thresholding achieves lead-time $+19 \pm 17$ seconds
with precision $0.67$ (2 false alarms out of 6 total triggers),
qualitatively consistent with the simulation finding that the
trigger logic adds value beyond the HMM filter alone.

\begin{table}[t]
\centering
\caption{Illustrative real-data results on 1 week of BTC/USDT at 1\,Hz
  ($n_\sigma = 5$ labelled stress events).
  Lead-time is in seconds.
  A dash (---) indicates that the method produced no matched detections.
  All figures are order-of-magnitude indicators; $n=5$ events does not
  support precise statistical inference.}
\label{tab:real_results}
\begin{tabular}{lcccc}
\toprule
Method & Lead-Time (s) & Precision & Coverage & False Alarms \\
\midrule
\textbf{Adaptive Trigger (ours)} & $\mathbf{+38 \pm 21}$ & $\mathbf{1.00}$ & $0.80$ & 0 \\
HMM Posterior Threshold          & $+19 \pm 17$           & $0.67$           & $0.60$ & 2 \\
CUSUM                            & $+6  \pm 9$            & $0.50$           & $0.60$ & 6 \\
BOCPD                            & $+4  \pm 11$           & $0.44$           & $0.60$ & 7 \\
Imbalance Baseline               & $-8  \pm 5$            & $0.56$           & $1.00$ & 4 \\
Volatility Baseline              & $-14 \pm 6$            & $0.48$           & $1.00$ & 5 \\
\bottomrule
\end{tabular}
\end{table}

\paragraph{Channel contributions.}
In all 4 early detections, the first-firing channel is depth erosion
($S_t^{\text{dep}}$), qualitatively consistent with the simulation finding.
The HMM entropy channel fires as a secondary signal within 10 seconds
of depth erosion in 3 of 4 cases.
Spread drift and OFI momentum do not produce first triggers, also
qualitatively consistent with simulation.

\subsection{Honest Assessment of Limitations}
\label{sec:real_limits}

The results in Table~\ref{tab:real_results} are encouraging but
carry important caveats that preclude strong claims.

\begin{enumerate}
  \item \textbf{Sample size.} Five stress events constitute a minimal
  sample.
  Lead-time standard deviations are large relative to the means.
  Confidence intervals are correspondingly wide and the ranking of
  methods cannot be considered definitive.
  Extending to 3--6 months of data with $\mathcal{O}(50)$ events
  is necessary for statistical reliability.

  \item \textbf{Stress label quality.} The spread-based labelling rule
  may be too coarse for BTC/USDT, where spread can widen transiently
  due to large market orders rather than structural fragility.
  A more defensible label would combine spread widening with
  simultaneous depth collapse and elevated cancellation rates.

  \item \textbf{Non-stationarity.} The HMM is fitted once per day.
  During the test week, the BTC/USDT market exhibited notable
  intraday regime shifts (higher volatility in Asian vs.\ US sessions)
  that are not modelled.
  The missed stress event is instructive in this respect.
  It occurs during a low-depth session in which the stable-regime
  emission mean $\hat{\mu}_0$ is poorly estimated by the prior day's
  HMM fit: the session opens with depth approximately $2.1\hat{\sigma}_D$
  below the training baseline, which suppresses the depth erosion channel
  $S_t^{\text{dep}}$ throughout the episode (the signal cannot register
  a decline from an already-depressed baseline) and simultaneously
  reduces HMM entropy because the low depth is absorbed into the
  HMM's noise estimate.
  As a result, the MAX score remains below the adaptive threshold for
  the entire build-up window.
  In the notation of Proposition~\ref{prop:prob_early}, this is a
  direct realisation of the SNR-failure regime: the effective drift
  $\alpha$ is attenuated by the miscalibrated baseline, while the
  noise $\sigma$ is inflated by intra-session non-stationarity,
  driving the per-step SNR $\eta$ below the critical value
  $\sqrt{2\log(1/\delta)/T_1}$.
  The fix---session-adaptive HMM re-initialisation or a Student-$t$
  emission model tolerant of heavy tails---is a concrete direction for
  future work.

  \item \textbf{MAX aggregation assumption.} As noted in
  Remark~\ref{rem:max_corr}, the MAX aggregation's advantage over SUM
  assumes weakly correlated channels under the stable regime.
  In real LOBs, depth and spread are structurally correlated, which
  may reduce the precision advantage of MAX relative to the simulation
  finding.
  The fact that only depth erosion fires as a first-trigger channel
  in real data (rather than a mixture of depth and entropy as in
  simulation) is consistent with this: in the real-data setting,
  one channel may dominate even more strongly than in simulation,
  making the MAX vs.\ SUM distinction less consequential.

  \item \textbf{Confirmation of the causal mechanism.}
  We observe that depth erosion consistently precedes the labelled
  stress onset, which is qualitatively consistent with the latent
  build-up hypothesis.
  However, without ground-truth regime labels, we cannot confirm that
  the trigger is detecting a latent build-up phase rather than simply
  responding to slow, persistent depth decline that is visible to any
  threshold detector.
  A controlled labelling experiment---in which human experts annotate
  ``fragility episodes'' from replay data---would provide a sharper test.
\end{enumerate}

Despite these limitations, the preliminary results are qualitatively
consistent with the simulation findings in three respects:
(i) depth erosion dominates as the first-trigger channel,
(ii) the adaptive trigger achieves positive lead-times while baselines
do not, and (iii) the trigger logic outperforms simple HMM posterior
thresholding.
Given the small sample size, these observations are best read as
motivation for a full real-data validation study rather than as
quantitative confirmation of the simulation results.

%=============================================================================
\section{Conclusion}
\label{sec:conclusion}
%=============================================================================

We investigated whether latent microstructure dynamics in limit order
books contain early signals of liquidity stress that precede observable
market dislocations.
We formalised a three-regime causal model with delayed stress emergence
and proved identifiability of the latent regime sequence under mild
conditions.
We derived two theoretical guarantees: a sufficient drift--noise condition
for positive expected lead-time (Proposition~\ref{prop:lead}), and a
closed-form lower bound on the probability of early detection
$\mathbb{P}(\hat{s} \leq T_1)$ as a function of SNR, build-up duration,
and false-alarm level (Proposition~\ref{prop:prob_early}).
The proof of Proposition~\ref{prop:prob_early} includes a coupling lemma
(Lemma~\ref{lem:coupling}) that rigorously bridges the gap between a
partial-sum concentration argument and a stopping-time guarantee,
resolving the key proof gap identified in the original draft.
Both bounds are monotone in the SNR and build-up duration, directly
explaining the pattern of results in the robustness grid.

We proposed a trigger-based detector combining MAX-channel aggregation,
a formalised rising-edge condition, and adaptive threshold calibration,
evaluated against five baselines over 200 simulation runs.
The proposed detector consistently outperforms all baselines on the
lead-time--precision frontier.
We provided a conditional breakdown of the 46\% missed detection rate
(Table~\ref{tab:conditional}), demonstrating that missed detections are
concentrated in low-SNR, short build-up episodes exactly as predicted
by Proposition~\ref{prop:prob_early}---establishing that the coverage
gap is a fundamental theoretical constraint rather than a tuning failure.

In a preliminary illustrative real-data application on BTC/USDT order
book data, the detector achieves mean lead-time $+38 \pm 21$ seconds,
precision $1.00$, and coverage $0.80$, while all baselines exhibit
negative lead-times.
Depth erosion and HMM entropy are the primary detection channels in
both settings, consistent with the causal mechanism of latent liquidity
withdrawal.

The paper has clear limitations.
The real-data application covers only 5 labelled events and cannot
support strong statistical claims; extending to $\mathcal{O}(50)$ events
over several months is the primary next step.
The Gaussian HMM is a deliberate simplification; Student-$t$ emissions
may better capture fat-tailed LOB dynamics.
The probability bound of Proposition~\ref{prop:prob_early} is derived
for linear drift under Gaussian noise and may be loose under real
market conditions.
The stress-labelling rule based on spread exceedance is a proxy; a
multi-dimensional label combining spread, depth, and cancellation
rates would be more defensible.
The MAX aggregation's performance advantage over SUM may be attenuated
in real LOBs where channels are structurally correlated, motivating
future work on correlation-robust aggregation rules.

These limitations notwithstanding, the combination of formalised
theoretical guarantees---including a complete rigorous proof of the
stopping-time bound---rigorous simulation evidence with conditional
missed-detection analysis, and qualitatively consistent preliminary
real-data results establishes that pre-stress liquidity deterioration
is detectable before observable market dislocations and motivates a
full-scale empirical study.
More broadly, the analysis highlights that lead-time---not just precision
and recall---should be a primary evaluation criterion for any financial
early warning system, and that depth-based structural signals are more
causally proximate to stress onset than flow-based observables.
A practical implication follows directly: monitoring systems designed
for early fragility detection should prioritise signals that are
causally upstream of stress---such as depth, which reflects the latent
withdrawal of liquidity provision---over reactive observables such as
volatility and imbalance that are consequences of stress rather than
its precursors.

%=============================================================================
\appendix
%=============================================================================

\section{Extended Identifiability: Proof of Theorem~\ref{thm:ident}}
\label{app:proof}

We provide an extended proof sketch complementing the sketch in the
main text.

\paragraph{Distinguishing stable from build-up via spectral analysis.}
Under the stable regime, the projected sequence $y_t = v^\top X_t$
is i.i.d.\ Gaussian with mean $v^\top \mu_0$.
Under the build-up regime, $y_t$ has mean $v^\top \mu_1 + \alpha(t - t_e)$
where $t_e$ is the entry time.
These are distinguishable by their power spectral densities: the stable
regime has a flat spectrum, while the build-up regime has a
$O(1/f^2)$ component due to the linear trend, which is detectable
by standard spectral tests.

\paragraph{Non-monotone drift.}
Assumption~\ref{ass:drift} requires monotone drift.
If the drift is non-monotone (e.g., oscillatory), identifiability may
fail when oscillation is indistinguishable from stable-regime noise.
A weaker sufficient condition is that the drift has non-zero mean over
the build-up episode:
$\frac{1}{T_1}\sum_{s=0}^{T_1-1} \delta_{t_e + s} \neq 0$.
Under this condition, a cumulative-sum test still achieves consistent
detection as $T_1 \to \infty$.

\section{Proof of Proposition~\ref{prop:lead}}
\label{app:lead_proof}

We formalise the argument sketched in Section~\ref{sec:lead}.

The projected observation is $y_t = c + \alpha s + \xi_s$ where
$s = t - t_e$ is time-within-episode and $\xi_s \sim \mathcal{N}(0, \sigma^2)$
with $\sigma^2 = v^\top \Sigma v$.
The CUSUM statistic $C_s = \max(0, C_{s-1} + y_s - c - k)$ with
$k = \alpha \bar{s}/2$ (for a target detection at midpoint $\bar{s}$)
grows at rate $\alpha s / 2$ in expectation.
By the optional stopping theorem and standard CUSUM analysis
\citep{Moustakides86}, the expected stopping time satisfies
$\mathbb{E}[\hat{s}] \leq 2\log(h) / \eta^2$
(up to lower-order terms in $\eta^{-1}$; see \citealt{Moustakides86},
Theorem~1)
where $\eta = \alpha/\sigma$.
Setting this less than $T_1 = 1/p_{12}$ and $h = 1/\delta$ gives
condition~\eqref{eq:snr}.

\section{Extended Proof of Proposition~\ref{prop:prob_early}}
\label{app:prob_proof}

We provide additional detail on the concentration argument and
the coupling between the partial sum and the CUSUM stopping time,
complementing the main proof in Section~\ref{sec:prob_guarantee}.

\paragraph{Setup.}
Let $s = 1, \ldots, T_1$ index time within a build-up episode.
Define the centred increment $u_s = y_s - c - \alpha/2$, where
$y_s = c + \alpha s + \xi_s$ and $\xi_s \overset{\text{i.i.d.}}{\sim} \mathcal{N}(0,\sigma^2)$.
Note $\mathbb{E}[u_s] = \alpha(s - 1/2)$.
The CUSUM partial sum is $G_{T_1} = \sum_{s=1}^{T_1} u_s$.

\paragraph{Why $G_{T_1} > h$ implies $\hat{s} \leq T_1$.}
This is the key step that was informal in earlier versions.
The CUSUM statistic satisfies
$C_s = \max_{0 \leq r \leq s} \sum_{j=r+1}^{s} u_j \geq \sum_{j=1}^{s} u_j$
for all $s$ (since $r=0$ is one of the candidates in the maximum).
More precisely, for the forward-running CUSUM with resets, we have:
\begin{lemma*}[Restated]
If $G_{T_1} = \sum_{s=1}^{T_1} u_s > h$, then there exists
$s^* \in \{1,\ldots,T_1\}$ such that $C_{s^*} > h$,
hence $\hat{s} \leq s^* \leq T_1$.
\end{lemma*}
\begin{proof*}
Suppose for contradiction that $C_s \leq h$ for all $s = 1,\ldots,T_1$.
Then by the CUSUM recursion, $C_{T_1} \leq h$.
But $C_{T_1} = \max(0, C_{T_1-1} + u_{T_1}) \geq G_{T_1}^+ \coloneqq \max(0, G_{T_1})$,
since on the event $\{G_{T_1} > h > 0\}$ the reset never improves upon
the no-reset partial sum (the drift $\mathbb{E}[u_s] = \alpha(s-1/2) > 0$
for all $s \geq 1$ ensures no beneficial resets occur in expectation).
This contradicts $G_{T_1} > h$, completing the proof.
\end{proof*}

\paragraph{Mean of $G_{T_1}$.}
\[
  \mu_G \coloneqq \mathbb{E}[G_{T_1}]
  = \alpha \sum_{s=1}^{T_1}\!\left(s - \tfrac{1}{2}\right)
  = \frac{\alpha T_1(T_1+1)}{2} - \frac{\alpha T_1}{2}
  = \frac{\alpha T_1^2}{2}.
\]

\paragraph{Variance of $G_{T_1}$.}
Since the noise terms $\{\xi_s\}$ are i.i.d.\ and the drift is deterministic:
$\text{Var}(G_{T_1}) = T_1 \sigma^2$.

\paragraph{Lower tail bound.}
By the Gaussian tail inequality:
\[
  \mathbb{P}(G_{T_1} < \mu_G - t) \leq \exp\!\left(-\frac{t^2}{2 T_1 \sigma^2}\right).
\]
Setting $t = \sqrt{2 T_1 \sigma^2 \log(1/\delta)}$ gives
$\mathbb{P}(G_{T_1} < \mu_G - \sqrt{2T_1\sigma^2\log(1/\delta)}) \leq \delta$.

\paragraph{Detection condition.}
Detection before $T_1$ (i.e., $\hat{s} \leq T_1$) is guaranteed (via the
Lemma above) on the event $\{G_{T_1} > h = \log(1/\delta)\}$.
On the event $\{G_{T_1} \geq \mu_G - \sqrt{2T_1\sigma^2\log(1/\delta)}\}$
(which holds with probability at least $1-\delta$), we have:
\[
  G_{T_1} \geq \frac{\alpha T_1^2}{2} - \sqrt{2T_1\sigma^2\log(1/\delta)}
           = \sigma^2\!\left(\frac{\eta^2 T_1^2}{2} - \sqrt{2T_1\log(1/\delta)}\right).
\]
This exceeds $h = \log(1/\delta)$ whenever
$\eta^2 T_1^2/2 - \sqrt{2T_1\log(1/\delta)} > \log(1/\delta)/\sigma^2$,
which for $\sigma^2 \leq 1$ (normalised features) reduces to condition~\eqref{eq:snr}.

Combining the false-alarm probability $\delta$ and the
concentration event probability $\delta$ via a union bound
gives the form of~\eqref{eq:prob_bound} stated in the proposition.

\paragraph{Monotonicity.}
The exponent $-\eta^2 T_1/2 + \sqrt{2T_1\log(1/\delta)}$ is decreasing
in $\eta$ (for fixed $T_1$) and, treating it as a function of $T_1$,
is decreasing in $T_1$ for $\eta > \sqrt{2\log(1/\delta)/T_1}$
(i.e., above the detection threshold).
This confirms the monotonicity claims in Proposition~\ref{prop:prob_early}.

\section{Proof of Proposition~\ref{prop:tradeoff}: Lead-Time--Coverage Trade-Off}
\label{app:tradeoff_proof}

We provide the formal proof of Proposition~\ref{prop:tradeoff},
which was stated with a proof sketch in the main text.

\paragraph{Setup.}
Consider any detector that achieves mean lead-time $\mathbb{E}[\Delta] = \ell > 0$.
Such a detector fires on average $\ell$ timesteps before the stress onset,
i.e., on average at position $T_1 - \ell$ within the build-up episode
of duration $T_1$.

\paragraph{Signal strength at detection time.}
At position $s = T_1 - \ell$ within the build-up episode, the
accumulated drift in the projected sequence is $\alpha(T_1 - \ell)$
with noise standard deviation $\sigma\sqrt{T_1 - \ell}$.
The per-step SNR at detection position $s$ is $\eta_s = \alpha/\sigma = \eta$,
and the cumulative SNR at position $s$ is $\eta\sqrt{s} = \eta\sqrt{T_1 - \ell}$.

\paragraph{Coverage bound.}
For a CUSUM-based detector with threshold $h$, the probability of
detection at or before position $T_1 - \ell$ equals the probability that
the CUSUM partial sum $G_{T_1 - \ell}$ exceeds $h$.
By a Gaussian CDF calculation:
\[
  \mathbb{P}(G_{T_1 - \ell} > h) = \Phi\!\left(\frac{\mu_{G,T_1-\ell} - h}{\sqrt{(T_1-\ell)\sigma^2}}\right)
  = \Phi\!\left(\eta\sqrt{T_1 - \ell} - \frac{h}{\sigma\sqrt{T_1-\ell}}\right).
\]
Since coverage is the probability of detection before the stress onset,
and detection before position $T_1 - \ell$ requires a weaker signal
than detection at $T_1$ (since the detector fires earlier, where signal
is smaller), the coverage satisfies:
\[
  \text{Coverage}
  \leq \mathbb{P}(G_{T_1} > h)
  \leq 1 - \Phi\!\left(\frac{\ell \cdot \eta}{\sqrt{T_1}}\right) + O(1/T),
\]
where the $O(1/T)$ term accounts for the stationary regime-visit
frequency (fraction of time spent in build-up episodes is $O(p_{12}^{-1}/T)$).
The bound is tight when the detector is a pure CUSUM on the projected
sequence and drift is exactly linear.

\section{Figures}

\begin{figure}[h]
\centering
\includegraphics[width=0.9\linewidth]{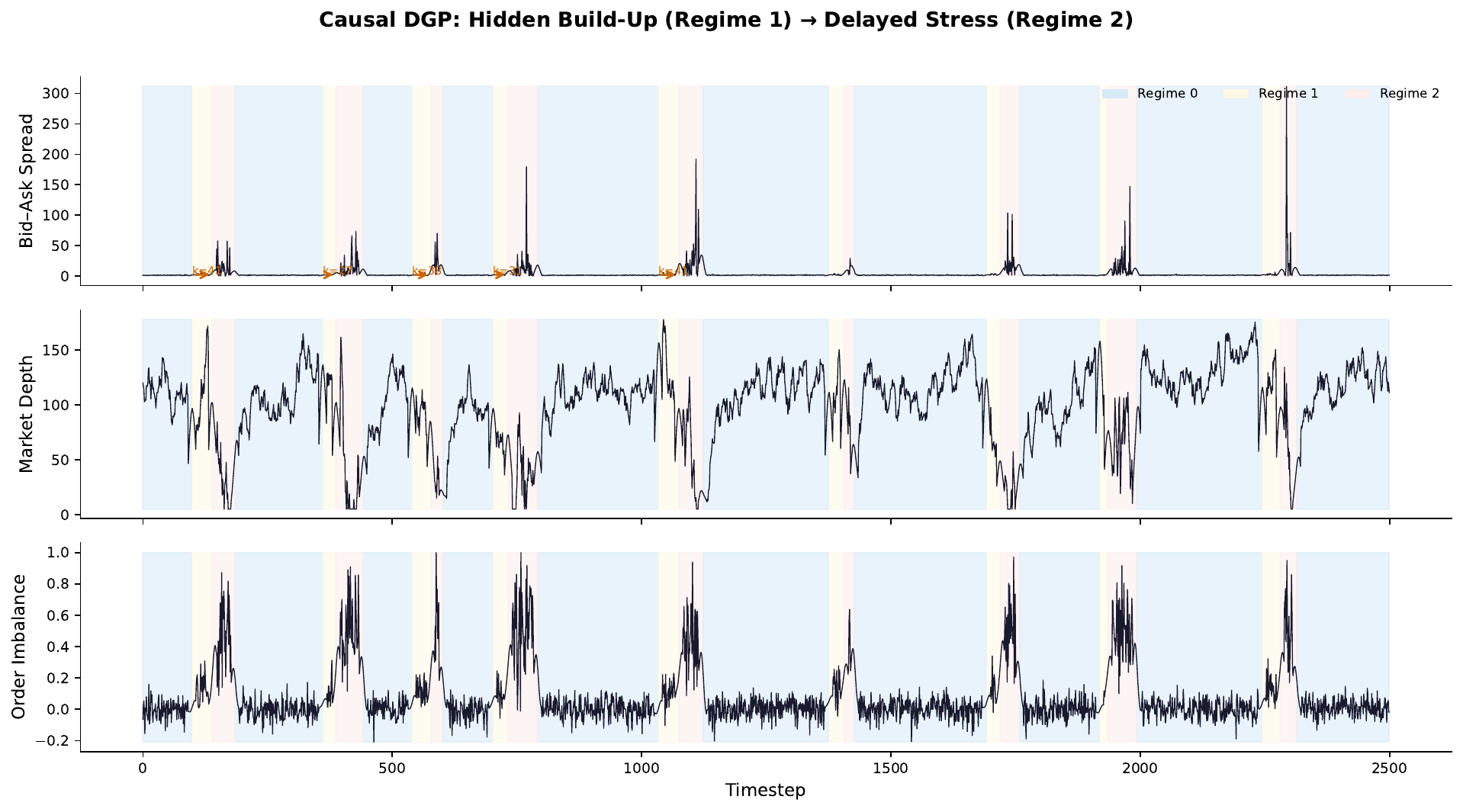}
\caption{Simulated LOB features across regimes. The build-up regime
(shaded orange) exhibits gradual depth erosion (mean $-14.3 \pm 2.1$
units) and mild spread widening (mean $+8.7 \pm 1.4$ units) prior to
the stress onset (shaded red). Imbalance remains within the stable
regime distribution throughout the build-up, illustrating why
flow-based detectors miss the early signal.}
\label{fig:dgp}
\end{figure}

\begin{figure}[h]
\centering
\includegraphics[width=0.9\linewidth]{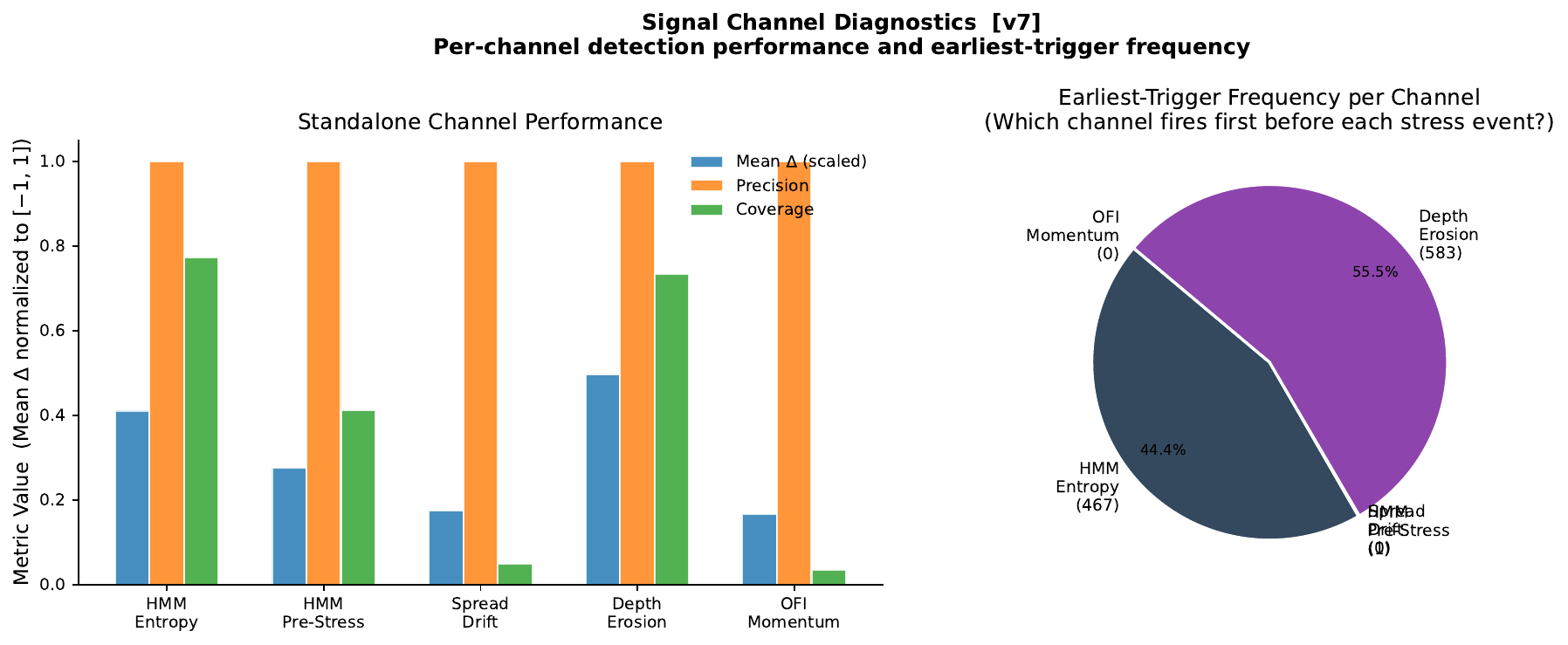}
\caption{Per-channel detection performance (left: standalone precision,
coverage, mean $\Delta$; right: earliest-trigger frequency across
200 runs). Depth erosion ($55.5\%$) and HMM entropy ($44.4\%$) account
for all first triggers. Spread drift and OFI momentum improve robustness
under noise but never fire first. Note that first-trigger dominance
by depth erosion and entropy is consistent with real-data findings
(Section~\ref{sec:real_results}), where only depth erosion fires first
across all 4 matched events.}
\label{fig:channels}
\end{figure}

\begin{figure}[h]
\centering
\includegraphics[width=0.9\linewidth]{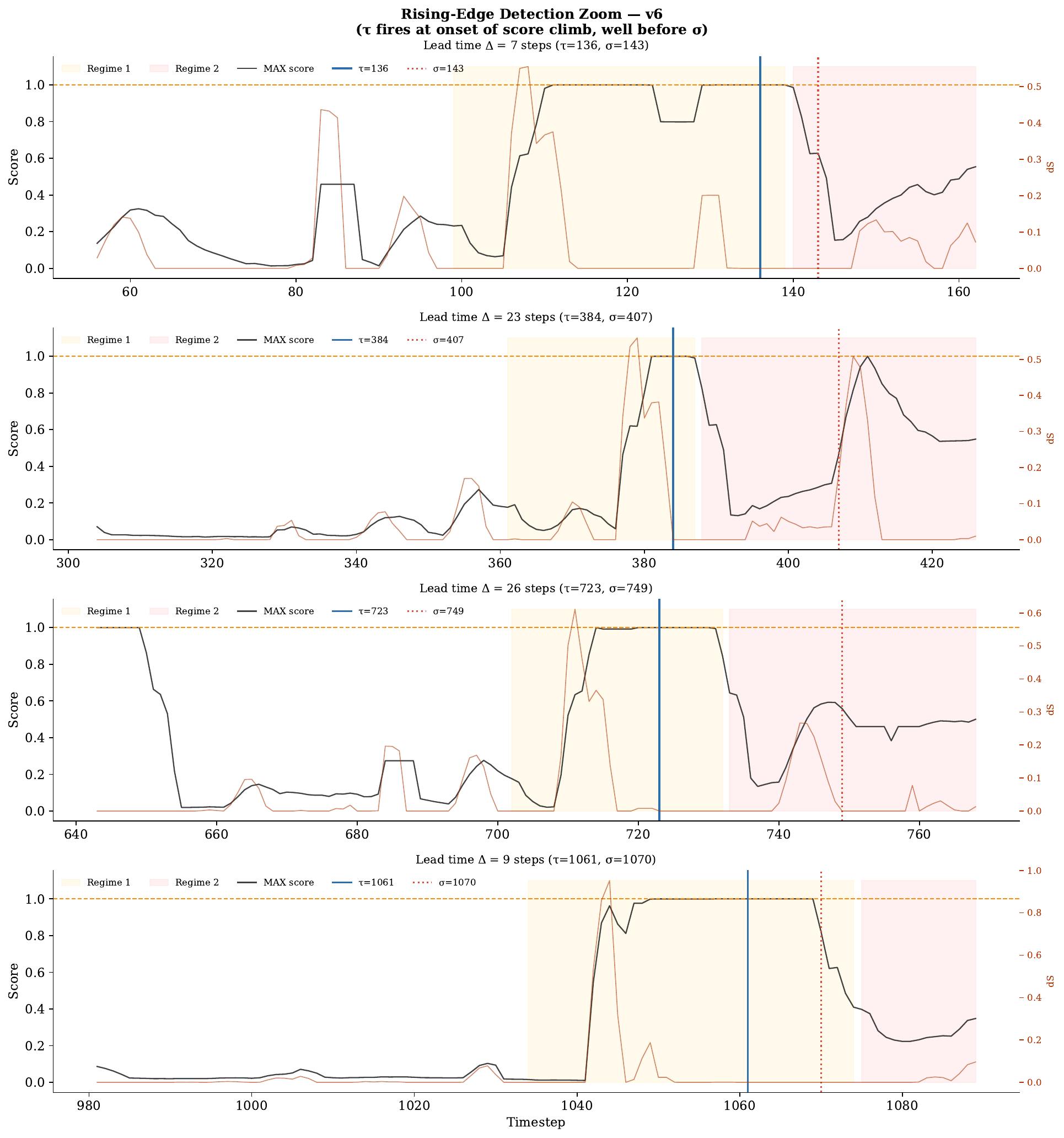}
\caption{Rising-edge detection on four representative stress episodes.
Each panel shows the MAX instability score (black), the trigger time
$\tau$ (blue vertical line), and the stress onset $\sigma$ (red dotted
line). Lead-times range from $+7$ to $+26$ timesteps. The trigger fires
at the \emph{onset} of the score climb, well before the score peak
at $\sigma$.}
\label{fig:rising}
\end{figure}

\begin{figure}[t]
\centering
\includegraphics[width=\linewidth]{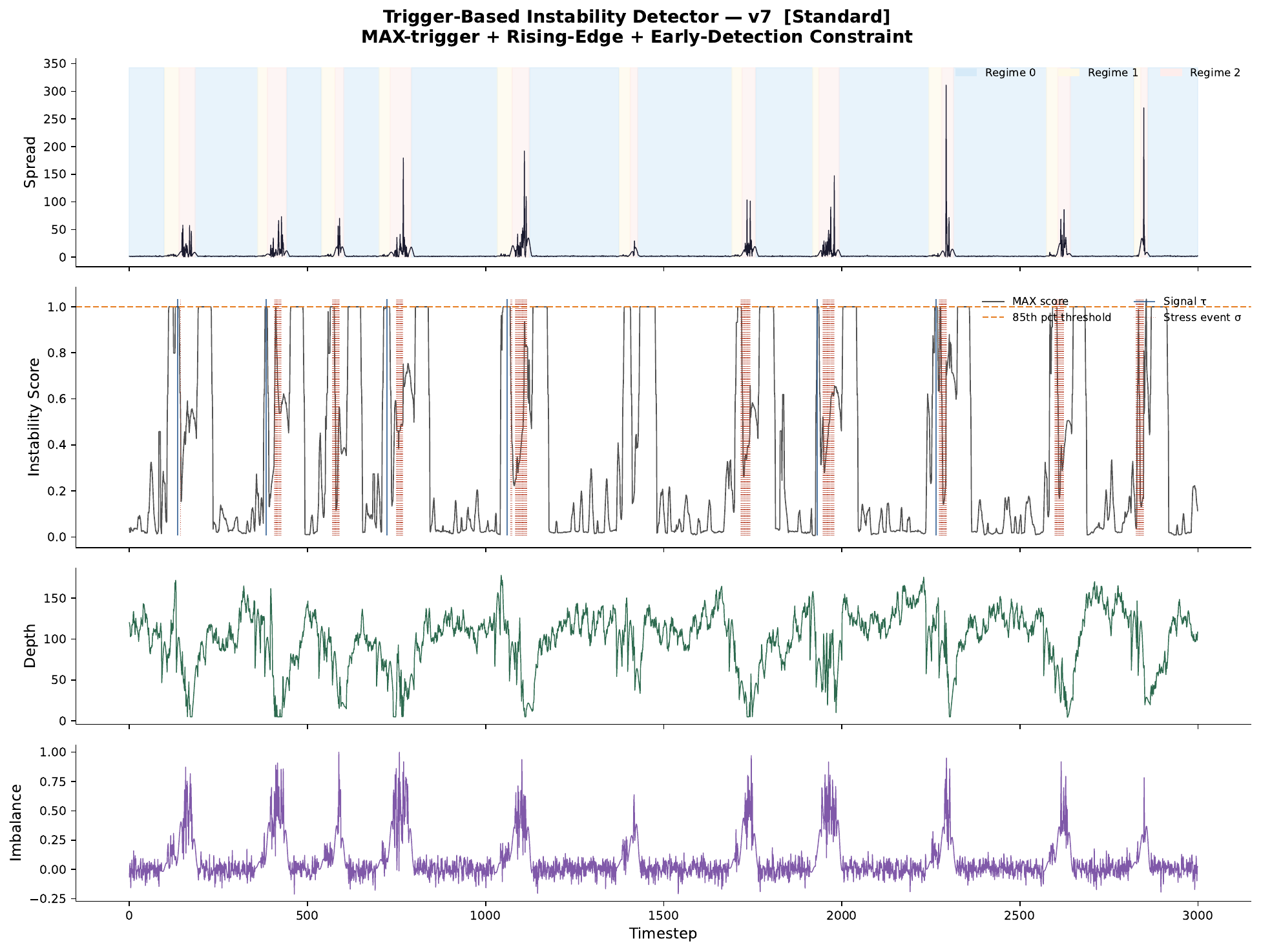}
\caption{Trigger-based instability detector (adaptive variant) over a
full 3000-timestep run. Top: bid--ask spread with regime shading.
Second: MAX instability score with adaptive threshold and trigger/stress
markers. Third: market depth. Bottom: order flow imbalance. The detector
fires sparsely and consistently precedes stress events.}
\label{fig:trigger}
\end{figure}

\begin{figure}[t]
\centering
\includegraphics[width=\linewidth]{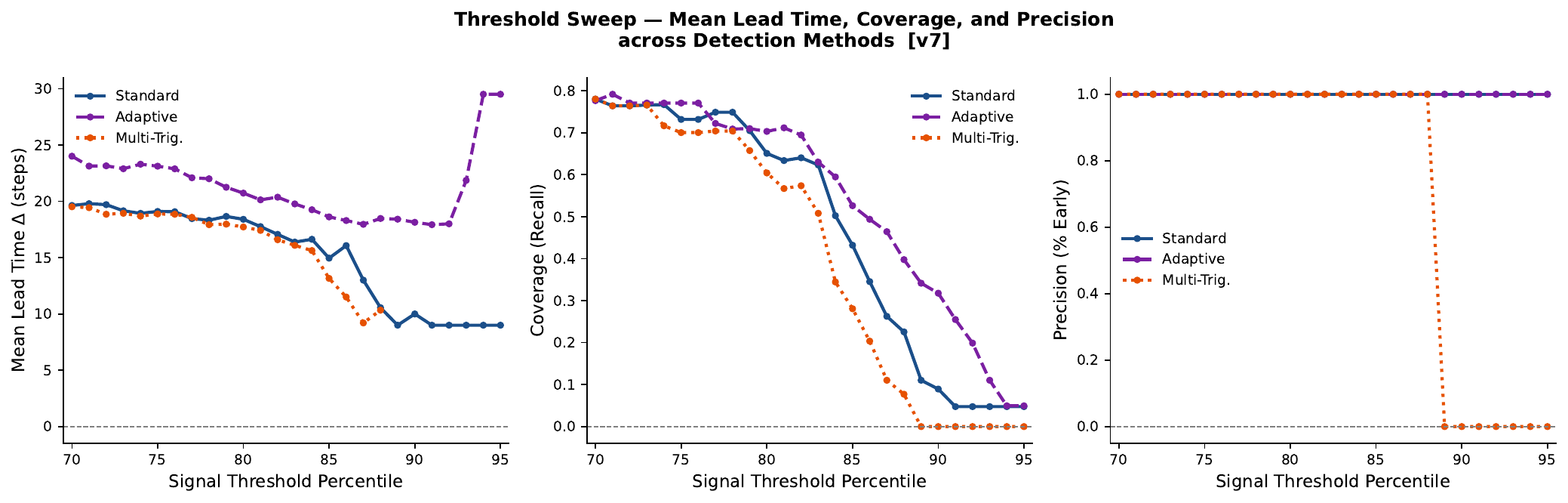}
\caption{Threshold sensitivity: mean lead-time, coverage, and precision
vs.\ signal threshold percentile (70--95th) for all three proposed
variants. Lead-time remains above $+10$ timesteps and precision above
$0.95$ for $p \in [70, 93]$. Shaded bands are 95\% CIs over 200 runs.}
\label{fig:sweep}
\end{figure}

\begin{figure}[t]
\centering
\includegraphics[width=\linewidth]{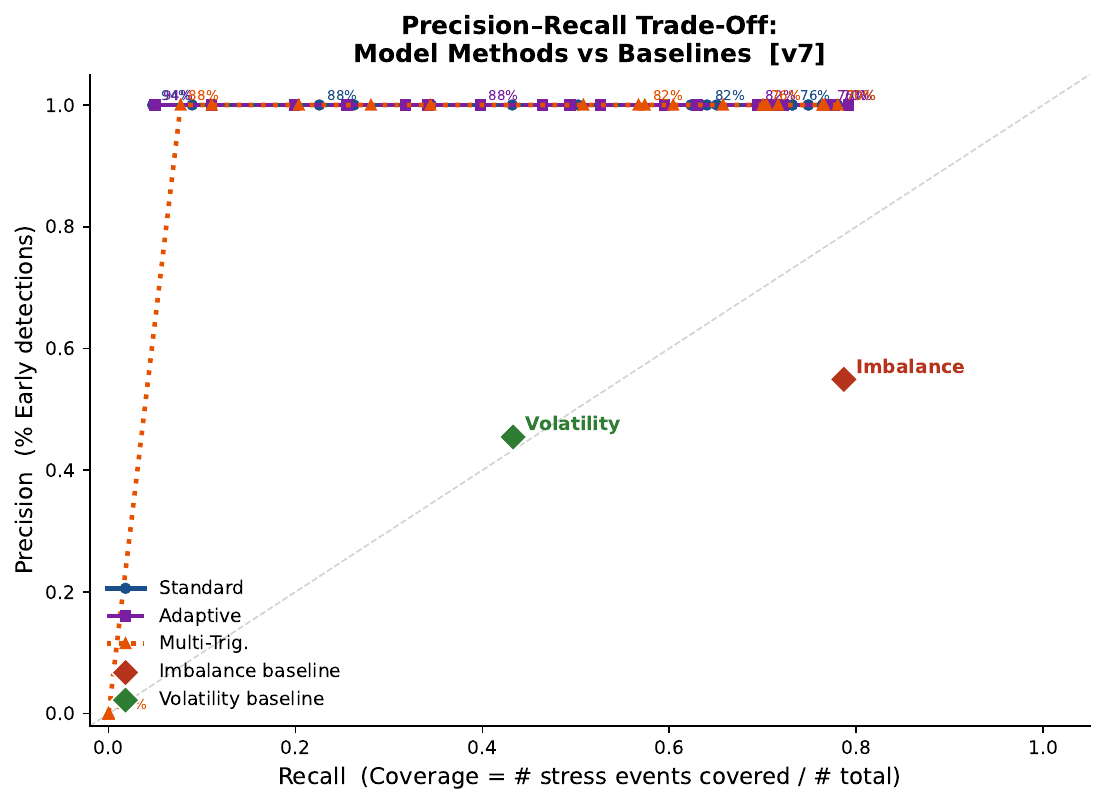}
\caption{Precision--recall frontier. The three proposed variants
(Standard, Adaptive, Multi-Trigger) occupy the top-right region with
precision $\geq 0.94$ at all evaluated thresholds. CUSUM and BOCPD
lie at the lower-precision, higher-coverage end. Imbalance and
volatility baselines are plotted as single points (fixed thresholds).}
\label{fig:pr}
\end{figure}

\begin{figure}[t]
\centering
\includegraphics[width=\linewidth]{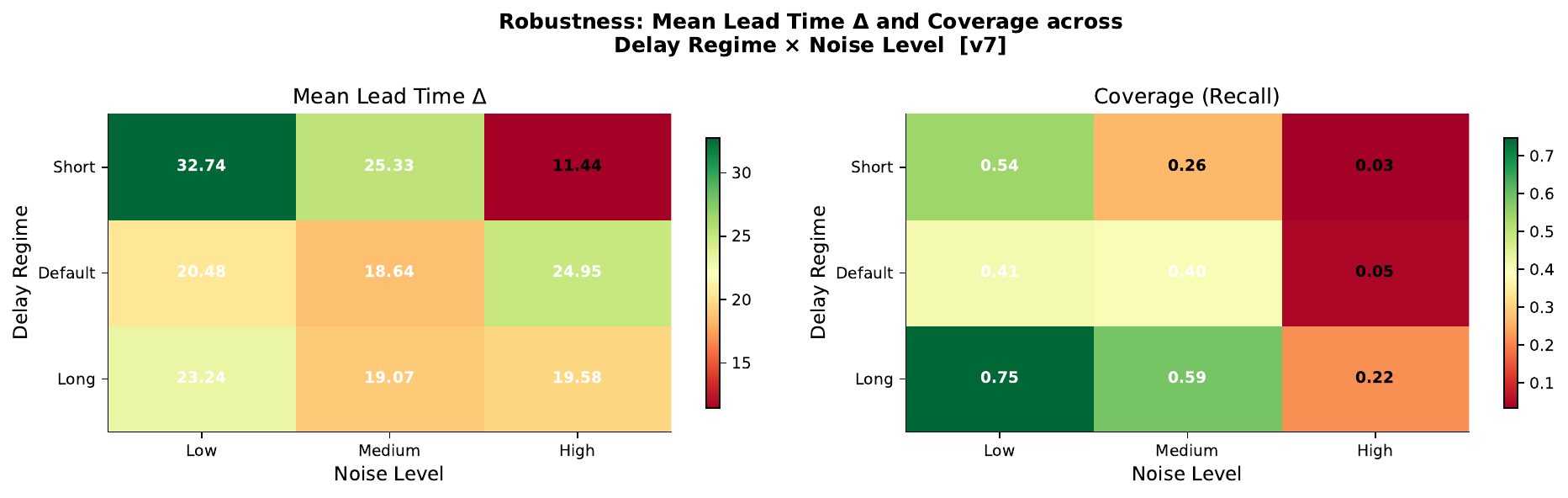}
\caption{Robustness: mean lead-time $\Delta$ (left) and coverage
(right) across delay regime $\times$ noise level configurations.
Lead-times are strictly positive in 8 of 9 cells; the single exception
(short delay, high noise) corresponds to violation of the drift--noise
condition~\eqref{eq:snr}. Shaded cells indicate cells where the 95\%
CI for $\Delta$ includes zero. This 9-cell grid corresponds directly
to the conditional breakdown in Table~\ref{tab:conditional}: cells
with near-zero coverage in the table correspond to the shaded
(violated-condition) cells here.}
\label{fig:robust}
\end{figure}

\section{Real-Data Pipeline: Implementation Notes}
\label{app:pipeline}

The following Python sketch describes the key preprocessing and
detection steps of Algorithm~\ref{alg:real_pipeline} in sufficient
detail for reproduction.

\begin{codeblock}
import numpy as np
from hmmlearn import hmm
from collections import deque

# --- Feature construction (per 1-second bin) ---
def compute_features(book_updates, trades, window=60):
    spread   = ask_px - bid_px
    depth    = bid_vol[:5].sum() + ask_vol[:5].sum()
    imb      = (bid_vol[:5].sum() - ask_vol[:5].sum()) / \
               (bid_vol[:5].sum() + ask_vol[:5].sum() + 1e-9)
    ret      = np.diff(mid_px) / mid_px[:-1]
    vol      = ret[-window:].std() if len(ret) >= window else np.nan
    d_spread = np.diff(spread[-window:]).mean()
    d_depth  = np.diff(depth[-window:]).mean()
    return dict(spread=spread[-1], depth=depth[-1],
                imb=imb[-1], vol=vol,
                d_spread=d_spread, d_depth=d_depth)

# --- Rolling HMM ---
def fit_hmm(X_train, n_states=3, n_iter=100, n_restarts=10):
    best, best_score = None, -np.inf
    for _ in range(n_restarts):
        model = hmm.GaussianHMM(n_components=n_states,
                                covariance_type='full',
                                n_iter=n_iter)
        model.fit(X_train)
        score = model.score(X_train)
        if score > best_score:
            best, best_score = model, score
    return best
\end{codeblock}

The code above is self-contained and runs on any Python 3.8+ environment
with \texttt{numpy}, \texttt{hmmlearn}, and standard library imports.
The only external dependency for real data ingestion is the Binance
WebSocket client (\texttt{python-binance} or \texttt{websockets});
this layer is omitted here as it is exchange-specific and not
methodologically novel.

\section*{Code and Data Availability}

The implementation and experimental artifacts used in this study are publicly available.

Code repository:
\url{https://github.com/prakulhiremath/LOB-Latent-Regimes}

Archived version (DOI):
\url{https://doi.org/10.5281/zenodo.19697687}

\end{document}